\documentclass[letterpaper]{article} 
\usepackage{aaai2027}  
\usepackage[hyphens]{url}  
\usepackage{graphicx} 
\usepackage{natbib}  
\usepackage{caption} 
\usepackage{algorithm}
\usepackage{algorithmic}

\usepackage{amsfonts}
\usepackage{amsmath}
\usepackage{amsthm}
\usepackage{amssymb}

\newtheorem{theorem}{Theorem}

\newtheorem{lemma}{Lemma}[section]
\newtheorem{corollary}[lemma]{Corollary}

\usepackage{newfloat}
\usepackage{listings}
\DeclareCaptionStyle{ruled}{labelfont=normalfont,labelsep=colon,strut=off} 
\floatstyle{ruled}
\newfloat{listing}{tb}{lst}{}
\floatname{listing}{Listing}

\usepackage{booktabs}

\title{Event-Structured Physics-Informed Neural Networks for Differentiable Critical Clearing Boundaries}
\author{
    Baoli Hao\textsuperscript{\rm 1},
    Chenxi Hu\textsuperscript{\rm 2},
    Ming Zhong\textsuperscript{\rm 3},
   Ren Wang\textsuperscript{\rm 4}
}
\affiliations{
    \textsuperscript{\rm 1} Department of Applied Mathematics, Illinois Institute of Technology, 
Chicago, IL 60616, USA \\ { bhao2@hawk.illinoistech.edu}
\\
\textsuperscript{\rm 2} Department of Computer Science, The Hang Seng University of Hong Kong, Sha Tin, Hong Kong \\ {chenxihu@hsu.edu.hk}
\\
\textsuperscript{\rm 3} Department of Mathematics, University of Houston, Houston TX 77204, USA \\
{ mzhong3@central.uh.edu}
\\
 \textsuperscript{\rm 4} Department of Electrical and Computer Engineering, Illinois Institute of Technology, 
Chicago, IL 60616, USA \\ { rwang74@illinoistech.edu} 

}

\nocopyright
\begin{document}

\maketitle

\begin{abstract}

Transient-stability assessment determines whether a power system can recover after a disturbance and is therefore essential to preventing generator trips and cascading outages. A key metric is the critical clearing time (CCT), which specifies the maximum time available to clear a fault before synchronism is lost.  Reliable CCT estimation is challenging because complicated fault-clearing dynamics require repeated simulations over many fault severities and clearing times.
 We propose an event-structured physics-informed neural network (ES-PINN) that aligns its representation with the pre-fault, fault-on, and post-clearing swing dynamics and enforces exact state chaining across event interfaces. A smooth trajectory-induced stability margin defines a differentiable approximation of the CCT boundary, enabling accurate boundary extraction, local sensitivity analysis, and optional direct CCT prediction through a distilled readout. We further prove a local residual-to-trajectory-to-CCT error
estimate, in which exact event chaining eliminates separate
state-interface defect terms. Experiments on IEEE 9-, 14-, and 30-bus systems show that ES-PINN consistently improves held-out trajectory and stability-boundary accuracy over matched neural-surrogate baselines across mechanical and electrical contingencies with multiple clearing configurations. Additional full-network DAE validation, multi-fault experiments, and runtime analyses further demonstrate the effectiveness and computational efficiency of the proposed framework.
 
\end{abstract}


\section{Introduction}

Transient stability assessment asks whether a power system preserves
synchronism after a large disturbance and its subsequent clearing. A
contingency induces a family of trajectories indexed by fault severity
and clearing time: fault application and clearing change the active
network and, in some cases, the mechanical input. The critical clearing
time (CCT), defined by the first stable-to-unstable transition as
clearing time increases, is therefore an operationally meaningful
boundary rather than a property of any single trajectory~\cite{yorino2010new}. Estimating
this boundary conventionally requires repeated time-domain simulations
over a fault--clearing grid
\cite{misyris2020physics,stiasny2023physics,moya2023dae}. This burden motivates a surrogate for the full parameterized trajectory family. Physics-informed neural networks (PINNs) offer an appealing way to
construct such a surrogate by incorporating governing equations into
trajectory learning \cite{raissi2019physics}. However,  fault--clearing trajectories introduce a structural difficulty
beyond temporal propagation in time-dependent PINNs \cite{wang2021understanding,hao2026stability}. The dynamics consist
of pre-fault, fault-on, and post-clearing regimes with different active
operators; the state is continuous at switching, but its derivative
generally is not. Moreover, the clearing interface moves with the
queried clearing time. A monolithic smooth surrogate must represent all
regimes simultaneously, whereas phase-wise surrogates coupled through
soft interface penalties can leave state defects that propagate into
the post-clearing response. Because the CCT is extracted from a
trajectory-induced stability transition, even small interface defects
can ultimately shift the estimated stability boundary.

We address this setting with an event-structured physics-informed
neural network (ES-PINN) for hybrid fault--clearing dynamics with
known event rules. ES-PINN
aligns its representation with the physical event sequence, using
separate pre-fault, fault-on, and post-clearing components that
are hard chained through the initial condition and phase interfaces. Thus, the
post-clearing state is obtained directly from the fault-on prediction
at the queried clearing time, preserving a differentiable trajectory
path with respect to that input. The primary output is the full
parameterized hybrid trajectory family, from which a smooth
rotor-angle-separation margin defines the stability boundary and the
 CCT. Implicit differentiation then yields the local CCT
sensitivity, while a post-hoc scalar head provides an optional direct
readout. Figure~\ref{fig:Framework} summarizes this trajectory-to-boundary
workflow.

\begin{figure*}[t]
    \centering
    \includegraphics[width=\linewidth]
    {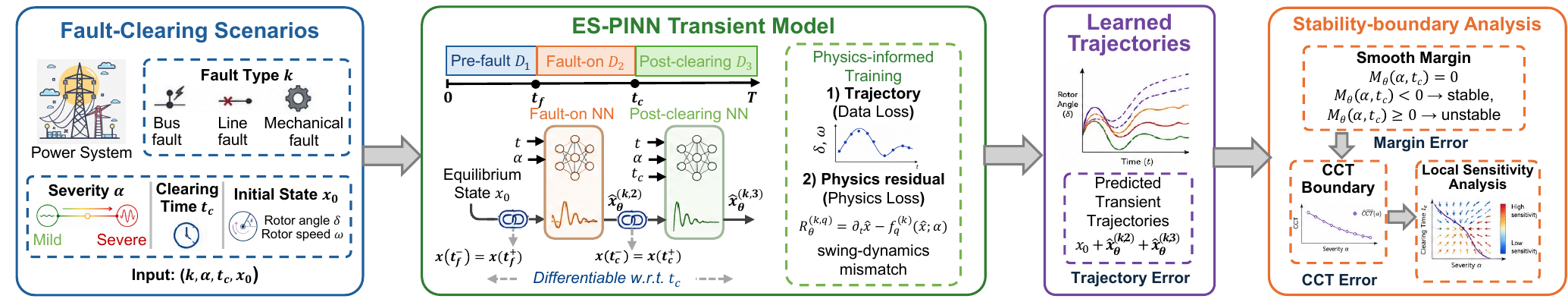}
    \caption{Flowchart of the ES-PINN Framework. Phase-wise networks are hard chained at the fault and clearing time and trained in a physics-informed manner. The learned trajectories induce a smooth margin whose zero level set gives the CCT boundary.}
    \label{fig:Framework}
\end{figure*}
Our main contributions are fourfold.
\textbf{(i)} We introduce ES-PINN, an event-structured surrogate that
hard chains phase-wise dynamics across an input-dependent clearing
interface.
\textbf{(ii)} We extract a differentiable CCT boundary and its local
sensitivity from the learned trajectory family, with an optional scalar
readout.
\textbf{(iii)} We prove a local theorem that provides a
residual-to-trajectory-to-CCT error estimate, linking physics residuals
and interface errors to trajectory and stability-boundary accuracy.
\textbf{(iv)} We evaluate held-out trajectories and CCT boundaries on
IEEE 9-, 14-, and 30-bus systems under mechanical and electrical faults,
multiple clearing actions, and limited-data multi-fault sharing.


\subsection{Related Work}

PINNs embed differential-equation residuals into neural-network
training \cite{raissi2019physics}, but their performance on
time-dependent problems can be limited by gradient imbalance and weak
propagation of initial information \cite{wang2021understanding}.
Recent enhancements to PINNs for time-dependent problems include exact
initial-condition enforcement, which improves training stability for
stiff systems~\cite{hao2026stability}, domain decomposition with
separate local networks in XPINN~\cite{jagtap2020extended}, and exact
temporal-continuity enforcement across sequential, fixed time
segments~\cite{roy2024exact}. Global operator-learning approaches,
including DeepONet and its physics-informed extension, instead learn
families of solutions through a shared operator
representation~\cite{lu2021learning,wang2021learning}.
ES-PINN differs in that its
clearing interface is an input-dependent event location: the
post-clearing state is generated from the fault-on trajectory at the
queried clearing time, while the  swing dynamics switch
across that interface.

Within power systems, PINNs have been used for dynamic state and
parameter estimation \cite{misyris2020physics}, time-domain simulation
\cite{stiasny2023physics}, DAE trajectory learning
\cite{moya2023dae}, and component-wise DAE simulation
\cite{stiasny2024pinnsim}. PINN-based CCT studies have combined learned
surrogates with optimization formulations
 \cite{misyris2021capturing,de2025critical}. 
In contrast, ES-PINN obtains CCT sensitivities by implicitly differentiating the learned trajectory-induced stability boundary rather than by solving the trajectory-sensitivity equations in~\cite{mishra2020critical}.
For heterogeneous
contingency families, we use known hard routing to share
event-structured representations, drawing on the general
mixture-of-experts principle \cite{shazeer2017outrageously}.


\section{Problem Setup}
\label{sec:problem}

We study transient-stability assessment for power-system contingencies
through reduced swing-equation dynamics. A contingency member
\(k\in\mathcal K\) specifies a known fault mechanism and location, its
clearing action, and the resulting pre-fault, fault-on, and
post-clearing network and input models. We represent the three
phase-wise operator pairs as
\(
c_k
=
(
(
Y^{(k,q)}(\alpha),
P_m^{(k,q)}(\alpha)
)
)_{q=1}^{3},
\)
where \(q=1,2,3\) denote the pre-fault, fault-on, and post-clearing
phases, respectively,
\(Y^{(k,q)}(\alpha)\in\mathbb C^{n_g\times n_g}\) is the reduced
generator-admittance matrix, and
\(P_m^{(k,q)}(\alpha)\in\mathbb R^{n_g}\) is the mechanical-power input
vector. Here, \(\alpha\in\mathcal A_k\subset\mathbb R\) is the
admissible severity parameter for contingency \(k\). The pre-fault and
post-clearing operators are independent of \(\alpha\) in the
configurations considered here. We instantiate \(c_k\) through several power-system contingency
families. Generator-bus and load-bus electrical faults are modeled by
adding an \(\alpha\)-scaled shunt admittance at the selected bus and
Kron-reducing the resulting network. A transmission-line fault is
modeled by inserting a virtual fault node at fractional line position
\(\lambda^{(k)}\in(0,1)\), applying an \(\alpha\)-scaled shunt at that
node, and Kron-reducing the augmented network. For a mechanical-power
perturbation at generator \(g\), the fault-on input is
\(
P_{m,g}^{(k,2)}(\alpha)
=
(
1+s_{\mathrm{mech}}^{(k)}\alpha
)
P_{m,g}^{(k,1)},
\)
where \(s_{\mathrm{mech}}^{(k)}\) is a dimensionless coefficient
controlling the relative change in the selected mechanical-power input;
the remaining generator inputs are unchanged. In each case,
\(\alpha=0\) recovers the pre-fault operator during the fault-on phase.
The clearing action may restore the pre-fault topology, remove a
selected transmission line, trip feeders incident to a faulted bus, or
restore a perturbed mechanical input. Thus, \(k\) identifies the fault mechanism, location, and clearing action, remaining fixed in the single-contingency setting and becoming
a routed discrete input in the multi-fault extension.  Let
\(
t_c\in\mathcal T
=
[t_{c,\min},t_{c,\max}]
\subset(t_f,T)
\)
denote the queried clearing time. As illustrated in
Fig.~\ref{fig:Framework}, the fault is applied at \(t_f\), cleared at
\(t_c\), and assessed over the horizon \(t\in[0,T]\). For a queried pair
\((\alpha,t_c)\in\mathcal A_k\times\mathcal T\), the reduced dynamic
state of a system with \(n_g\) generators is
\(
x^{(k)}(t;\alpha,t_c)
=
\left(
\delta^{(k)}(t;\alpha,t_c),
\omega^{(k)}(t;\alpha,t_c)
\right)
\in\mathbb R^{2n_g},
\)
where \(\delta\) is the rotor-angle vector and \(\omega\) is the
absolute per-unit rotor-speed vector. Within each event phase, the
reduced swing equations are
\begin{equation}
\begin{aligned}
\dot{\delta}
&=
\omega_b(\omega-\omega_0),
\\
\dot{\omega}
&=
(2H)^{-1}
\left[
P_m
-
P_e(\delta;Y)
-
D(\omega-\omega_0)
\right],
\end{aligned}
\label{eq:red_swing}
\end{equation}
where \(\omega_b\) is the base electrical angular frequency,
\(\omega_0\) is the synchronous-speed vector, and
\(H=\operatorname{diag}(H_1,\ldots,H_{n_g})\) and
\(D=\operatorname{diag}(D_1,\ldots,D_{n_g})\) are the inertia and
damping matrices, respectively. Writing
\(
Y_{ij}=|Y_{ij}|e^{\mathrm{i}\angle Y_{ij}},
\)
where \(\angle Y_{ij}\) is the phase angle of the reduced-admittance
entry, and letting \(E_i\) denote the constant internal-voltage
magnitude of generator \(i\), the electrical air-gap power is
\begin{equation}
P_{e,i}(\delta;Y)
=
\sum_{j=1}^{n_g}
E_iE_j |Y_{ij}|
\cos\!\left(
\delta_i-\delta_j-\angle Y_{ij}
\right).
\label{eq:Pei}
\end{equation}

Let
\(
(\tau_0,\tau_1,\tau_2,\tau_3)
=
(0,t_f,t_c,T).
\)
The hybrid dynamics can be written compactly as
\(
\dot{x}^{(k)}(t;\alpha,t_c)
=
f_q^{(k)}
\left(
x^{(k)}(t;\alpha,t_c);\alpha
\right),
\;
t\in[\tau_{q-1},\tau_q),
\)
where \(q=1,2,3\), and \(f_q^{(k)}\) is induced by the active operator
pair
\(
\bigl(
Y^{(k,q)}(\alpha),
P_m^{(k,q)}(\alpha)
\bigr).
\)
 We assume a
pre-contingency operating point
\(
x_0=(\delta_0,\omega_0)
\)
satisfying
\(
f_1^{(k)}(x_0)=0;
\)
hence, the pre-fault trajectory remains at \(x_0\) until fault
application.

The clearing time determines when the dynamics switch from the
fault-on to the post-clearing vector field; it does not enter the
pre-fault or fault-on dynamics. Therefore, for any
\begin{equation} x^{(k)}(t;\alpha,t_c^{(1)}) = x^{(k)}(t;\alpha,t_c^{(2)}), \; t< \min\{ t_c^{(1)},t_c^{(2)} \}. \label{eq:physical-nonanticipativity} \end{equation}
The generator dynamic state is also continuous across both switching
events:
\begin{equation}
\begin{aligned}
x^{(k)}(t_f^-;\alpha,t_c)
&=
x^{(k)}(t_f^+;\alpha,t_c),
\\
x^{(k)}(t_c^-;\alpha,t_c)
&=
x^{(k)}(t_c^+;\alpha,t_c).
\end{aligned}
\label{eq:state-continuity}
\end{equation}
Thus, the fault-on trajectory is non-anticipative and the event states
are continuous, although their time derivatives generally jump when
the active network or mechanical input changes. This event structure
motivates the phase-wise surrogate introduced below. We consider one
fault-application event and one clearing event per trajectory, with
known phase-wise operators and continuous generator states; unknown,
repeated, or state-reset switching mechanisms are outside the present
scope.

We measure loss of synchronism through the maximum pairwise rotor-angle
separation
\begin{equation}
d^{(k)}(t;\alpha,t_c)
\!=
\!\!\!\max_{1\leq i<j\leq n_g}\!
\left|
\delta_i^{(k)}(t;\alpha,t_c)
\!-\!
\delta_j^{(k)}(t;\alpha,t_c)
\right|,
\label{eq:angle-separation}
\end{equation}
which is invariant to a common shift of all rotor angles. The
continuous-time physical stability margin is
\begin{equation}
m^{(k)}(\alpha,t_c)
=
\max_{t\in[0,T]}
d^{(k)}(t;\alpha,t_c)
-
\Delta_{\max},
\label{eq:physical-hard-margin}
\end{equation}
where \(\Delta_{\max}>0\) is the prescribed maximum admissible
pairwise rotor-angle separation. A negative margin indicates
finite-horizon stability, whereas a nonnegative margin indicates loss
of synchronism over the assessment horizon \(T\), rather than
asymptotic instability beyond \(T\).

For a severity satisfying
\(
m^{(k)}(\alpha,t_{c,\min})<0,
\)
we define the critical clearing time as the right endpoint of the
initial stable branch:
\begin{equation}
\begin{aligned}
\mathrm{CCT}_k(\alpha)
=
\sup
\{
t_c\in\mathcal T\!\!:
&m^{(k)}(\alpha,s)<0,
\;
\forall s\in[t_{c,\min},t_c]
\}.
\end{aligned}
\label{eq:cct-definition}
\end{equation}
When the margin is continuous and the crossing is finite, this is the
first clearing time at which the margin becomes nonnegative, so any
later re-stabilized branch does not alter the reported CCT. If all
queried clearing times remain stable, the CCT is reported as
right-censored at \(t_{c,\max}\); if the trajectory is already unstable
at \(t_{c,\min}\), it is reported as left-censored at the lower
endpoint.

\section{Event-Structured PINN}
\label{sec:ecpinn}

The clearing time \(t_c\) is both an event location and a query
coordinate. ES-PINN assigns separate components to the pre-fault,
fault-on, and post-clearing domains while coupling adjacent components
through their physical event states. Specifically, these domains are
\(
\mathcal D_1=[0,t_f),
\)
\(
\mathcal D_2=[t_f,t_c),
\)
and
\(
\mathcal D_3=[t_c,T],
\)
respectively:
\begin{equation}
\widehat{x}^{(k)}_\theta(t;\alpha,t_c)
=
\begin{cases}
x_0,
& t\in\mathcal D_1,\\
\widehat{x}^{(k,2)}_\theta(t;\alpha),
& t\in\mathcal D_2,\\
\widehat{x}^{(k,3)}_\theta(t;\alpha,t_c),
& t\in\mathcal D_3.
\end{cases}
\label{eq:ecpinn-piecewise}
\end{equation}
Here
\(
\widehat{x}^{(k,q)}_\theta
=
(
\widehat{\delta}^{(k,q)}_\theta,
\widehat{\omega}^{(k,q)}_\theta
)
\)
for \(q\in\{2,3\}\).

For the two learned phases, let
\(
t_{\mathrm{bd}}^{(2)}=t_f,
\;
t_{\mathrm{bd}}^{(3)}=t_c,
\;
s_q=t-t_{\mathrm{bd}}^{(q)}.
\)
Their event states are hard chained as
\begin{equation}
\begin{aligned}
\widehat{x}^{(k,2)}_\theta(t_f;\alpha)
&=
x_{\mathrm{bd}}^{(k,2)}
=
x_0,
\\
\widehat{x}^{(k,3)}_\theta(t_c;\alpha,t_c)
&=
x_{\mathrm{bd}}^{(k,3)}(\alpha,t_c)
=
\widehat{x}^{(k,2)}_\theta(t_c;\alpha).
\end{aligned}
\label{eq:ecpinn-hard-chain}
\end{equation}
Thus, ES-PINN satisfies the state-continuity condition in
Eq.~\eqref{eq:state-continuity} by construction. Because the fault-on
component receives only \((t,\alpha)\), the physical
non-anticipativity condition in
Eq.~\eqref{eq:physical-nonanticipativity} also holds independently of
training accuracy. Moreover, evaluating the fault-on component at the
queried \(t_c\) preserves the differentiable dependency of the
post-clearing trajectory on the clearing time. The phase-3 boundary
state and its local physical coefficients therefore remain in the same
computational graph.

Write
\(
x_{\mathrm{bd}}^{(k,q)}
=
(
\delta_{\mathrm{bd}}^{(k,q)},
\omega_{\mathrm{bd}}^{(k,q)}
)
\)
and partition the phase-\(q\) vector field as
\(
f_q^{(k)}
=
(
f_{\delta,q}^{(k)},
f_{\omega,q}^{(k)}
).
\)
Let \(J_\delta f_{\omega,q}^{(k)}\) and
\(J_\omega f_{\omega,q}^{(k)}\) denote the Jacobians of the
rotor-speed vector field with respect to \(\delta\) and \(\omega\),
respectively. Differentiating the active swing dynamics along the
trajectory gives the local boundary coefficients
\begin{small}
\begin{equation}
\begin{aligned}
v_{\mathrm{bd}}^{(k,q)}
&=
f_{\delta,q}^{(k)}
\left(
x_{\mathrm{bd}}^{(k,q)};\alpha
\right),
\quad
a_{\mathrm{bd}}^{(k,q)}
=
f_{\omega,q}^{(k)}
\left(
x_{\mathrm{bd}}^{(k,q)};\alpha
\right),
\\
b_{\mathrm{bd}}^{(k,q)}
&=
J_\delta f_{\omega,q}^{(k)}
\left(
x_{\mathrm{bd}}^{(k,q)};\alpha
\right)
v_{\mathrm{bd}}^{(k,q)}
+
J_\omega f_{\omega,q}^{(k)}
\left(
x_{\mathrm{bd}}^{(k,q)};\alpha
\right)
a_{\mathrm{bd}}^{(k,q)},
\end{aligned}
\label{eq:ecpinn-taylor-coefficients}
\end{equation}
\end{small}
for $ q\in\{2,3\}.$
These coefficients give the phase-local swing-equation expansion
\begin{equation}
\begin{aligned}
T_{\delta}^{(k,q)}(s_q)
&=
\delta_{\mathrm{bd}}^{(k,q)}
+
v_{\mathrm{bd}}^{(k,q)}s_q
+
\frac{1}{2}
\omega_ba_{\mathrm{bd}}^{(k,q)}s_q^2
+
\frac{1}{6}
\omega_bb_{\mathrm{bd}}^{(k,q)}s_q^3,
\\
T_{\omega}^{(k,q)}(s_q)
&=
\omega_{\mathrm{bd}}^{(k,q)}
+
a_{\mathrm{bd}}^{(k,q)}s_q
+
\frac{1}{2}
b_{\mathrm{bd}}^{(k,q)}s_q^2,
\end{aligned}
\label{eq:ecpinn-local-taylor}
\end{equation}
with
\(
T^{(k,q)}
=
(
T_{\delta}^{(k,q)},
T_{\omega}^{(k,q)}
).
\)

Let
\(
\zeta_2=(t,\alpha)
\)
and
\(
\zeta_3=(t,\alpha,t_c)
\)
denote the causal inputs of the fault-on and post-clearing corrections,
respectively. We blend the local physical expansion with a neural
correction through the Taylor-gated ansatz $g_q(s_q)
\!=\!
1-
\exp[
-
\left(
\frac{s_q}{\ell_q}
\right)^2
],$
\begin{equation}
\begin{aligned}
\widehat{x}^{(k,q)}_\theta
=
&x_{\mathrm{bd}}^{(k,q)}
+
g_q(s_q)
A^{(k,q)}_\theta(\zeta_q)\\
&+
\left(
1-g_q(s_q)
\right)
\left(
T^{(k,q)}(s_q)
-
x_{\mathrm{bd}}^{(k,q)}
\right),
\end{aligned}
\label{eq:ecpinn-taylor-gated-ansatz}
\end{equation}
where $q\!\in\!\{2,3\},$ \(\ell_q\!>\!0\) is a fixed time scale independent of
\(\alpha\), \(t_c\), and the phase duration. Because
\(
g_q(0)=g_q'(0)=0,
\)
the ansatz fixes both the event state and the first right-hand time
derivative prescribed by the active phase dynamics. 
This preserves the physical jump in the time derivative at switching while preventing the neural correction from altering the initial derivative of the new phase. The Taylor expansion acts as a local physical inductive bias rather than a higher-order hard constraint, so the neural correction can adjust the higher-order derivatives away from the event.

Training combines trajectory supervision with phase-wise
swing-equation residuals. For \(q\!\in\!\{2,3\}\), define the 
residual
\begin{equation}
R_\theta^{(k,q)}
\! =\!
\partial_t
\widehat{x}^{(k,q)}_\theta
\!-\!
f_q^{(k)}
(
\widehat{x}^{(k,q)}_\theta;\alpha
)
\!=\!
(
R_{\delta,\theta}^{(k,q)},
R_{\omega,\theta}^{(k,q)}
).
\label{eq:ecpinn-phase-residuals}
\end{equation}
The training objective is
\begin{equation}
\mathcal L(\theta)
=
\mathcal L_{\mathrm{data}}
+
\lambda_\delta
\sum_{q=2}^{3}
\mathcal L_{\mathrm{res},\delta}^{(q)}
+
\lambda_\omega
\sum_{q=2}^{3}
\mathcal L_{\mathrm{res},\omega}^{(q)}.
\label{eq:ecpinn-loss}
\end{equation}
where \(\mathcal L_{\mathrm{data}}\) is the mean-squared trajectory-data loss, and the residual losses are mean-squared collocation residuals. The pre-fault residual vanishes identically
because \(\mathcal D_1\) is represented by the exact equilibrium
\(x_0\). Likewise, Eq.~\eqref{eq:ecpinn-hard-chain} imposes the initial
and interface states by construction, so no separate soft initial or
interface penalty is required.

\paragraph{Hard-routed multi-fault extension.}
For multiple known contingency members \(k\in\mathcal K\), we adopt a
deterministically routed mixture-of-experts architecture, inspired by
the  shared-representation and expert-specialization principle
~\cite{shazeer2017outrageously}.
Let \(H_\theta\) be a shared trunk, \(N_q\) a phase-specific neck, and
\(E_{k,q}\) a fault- and phase-specific expert. The correction in
Eq.~\eqref{eq:ecpinn-taylor-gated-ansatz} becomes
\begin{equation}
A^{(k,q)}_\theta(\zeta_q)
=
E_{k,q}
\left[
N_q
\left(
H_\theta(\zeta_q,q)
\right)
\right],
\quad q\in\{2,3\}.
\label{eq:ecpinn-hard-routed-moe}
\end{equation}
The known contingency index \(k\) selects exactly one expert during
both training and inference, with no soft expert averaging. The shared
trunk and phase-specific necks capture common transient structure,
while the routed heads retain fault-specific fault-on and
post-clearing responses. The exact pre-fault state bypasses the routed
neural architecture.


\section{Boundary Extraction and Deployment}
\label{sec:boundary}

We induce the stability boundary directly from the learned rotor-angle
trajectories, without training a separate stability classifier. Let
\(
\mathcal T_h=\{t_\ell\}_{\ell=1}^{N_t}
\)
denote the trajectory-time samples used for boundary evaluation, with
maximum spacing \(h_t\). For \(\beta>0\), define
\(
\operatorname{smax}_{\beta}\{z_m\}
=
\beta^{-1}\log\bigg(\sum_m\exp(\beta z_m)\bigg),
\)
whose error relative to \(\max_m z_m\) is at most
\(\log N/\beta\). We also use
\(
|z|_\varepsilon=\sqrt{z^2+\varepsilon^2},
\)
which satisfies
\(
0\leq |z|_\varepsilon-|z|\leq\varepsilon.
\)

Applying these approximations jointly over the trajectory-time samples
and unordered generator pairs gives
\begin{equation}
\mathcal G_{\beta,\varepsilon,h_t}[\delta]
=
\operatorname{smax}_{\beta}
\left\{
\left|
\delta_i(t_\ell)-\delta_j(t_\ell)
\right|_\varepsilon
:
t_\ell\in\mathcal T_h,\ i<j
\right\},
\label{eq:smooth-separation}
\end{equation}
where
\(
N=N_t\binom{n_g}{2}
\)
is the number of aggregated terms. The sampled-smooth margins of the
reduced ODE trajectory and its learned approximation are
\begin{equation}
\begin{aligned}
M^{(k)}(\alpha,t_c)
&=
\mathcal G_{\beta,\varepsilon,h_t}
\left[
\delta^{(k)}(\cdot;\alpha,t_c)
\right]
-
\Delta_{\max},
\\
M_\theta^{(k)}(\alpha,t_c)
&=
\mathcal G_{\beta,\varepsilon,h_t}
\left[
\widehat{\delta}^{(k)}_\theta
(\cdot;\alpha,t_c)
\right]
-
\Delta_{\max}.
\end{aligned}
\label{eq:smooth-margin}
\end{equation}
Both margins are induced by rotor-angle trajectories rather than by an
independently trained stability classifier.

If all pairwise rotor-angle separations are \(K_t\)-Lipschitz in time
on the relevant trajectory family, the ODE sampled-smooth margin
approximates the continuous-time hard margin in
Eq.~\eqref{eq:physical-hard-margin} according to
\begin{equation}
\left|
M^{(k)}(\alpha,t_c)
-
m^{(k)}(\alpha,t_c)
\right|
\leq
\varepsilon
+
\frac{\log N}{\beta}
+
K_t h_t.
\label{eq:smooth-hard-margin-error}
\end{equation}
The three terms quantify the smooth-absolute, smooth-maximum, and
temporal-sampling errors, respectively.

For each severity \(\alpha\), we evaluate
\(M_\theta^{(k)}(\alpha,t_c)\) over an ordered clearing-time grid
\(
t_{c,1}<\cdots<t_{c,N_c}.
\)
We select the first adjacent pair satisfying
\(
M_\theta^{(k)}(\alpha,t_{c,r})<0,
\;
M_\theta^{(k)}(\alpha,t_{c,r+1})\geq0,
\)
which brackets the initial stable-to-unstable transition. The learned
margin is then reevaluated at continuous clearing-time queries within
this bracket. Let \(t_\theta^{\circ,(k)}(\alpha)\) denote the selected
continuous zero, and let
\(\widehat{\mathrm{CCT}}_\theta^{(k)}(\alpha)\) denote its numerically
root-refined approximation:
\begin{equation}
\begin{aligned}
M_\theta^{(k)}
\left(
\alpha,t_\theta^{\circ,(k)}(\alpha)
\right)
&=
0,
\\
\left|
\widehat{\mathrm{CCT}}_\theta^{(k)}(\alpha)
-
t_\theta^{\circ,(k)}(\alpha)
\right|
&\leq
\tau_{\mathrm{root}}.
\end{aligned}
\label{eq:learned-margin-root}
\end{equation}
Censoring follows Sec.~\ref{sec:problem}, and any later re-stabilized
branch is excluded by selecting the first transition.

\begin{theorem}[Uniform residual-to-hard-CCT control]
\label{thm:end-to-end-cct}
Fix a contingency member \(k\), and let
\(\mathcal A_0\subseteq\mathcal A_k\) be a compact severity interval on
which the reduced-model hard CCT is finite and uncensored. Assume that
the phase-wise vector fields are uniformly Lipschitz on a common
bounded trajectory tube. Let
\(\mathcal R_2^{(k)}\) and \(\mathcal R_3^{(k)}\) denote the uniform
continuous \(L^1\) residual norms of the fault-on and post-clearing
phases associated with Eq.~\eqref{eq:ecpinn-phase-residuals}. Then
there exist constants \(C_2,C_3>0\), depending only on the phase-wise
Lipschitz constants and maximum phase durations, such that
\begin{small}
\begin{equation}
\sup_{\substack{
\alpha\in\mathcal A_0,\,
t_c\in\mathcal T\\
t\in[0,T]
}}
\left\|
\widehat{x}^{(k)}_\theta(t;\alpha,t_c)
-
x^{(k)}(t;\alpha,t_c)
\right\|
\leq
C_2\mathcal R_2^{(k)}
+
C_3\mathcal R_3^{(k)}.
\label{eq:trajectory-control}
\end{equation}
\end{small}

Assume further that the initial hard-CCT branch is uniformly interior,
isolated from earlier crossings, and transverse with margin-growth
constant \(\mu>0\). If the combined trajectory, margin-approximation,
and clearing-grid errors are sufficiently small to preserve this
initial branch, then
\begin{small}
\begin{equation}
\begin{aligned}
&\sup_{\alpha\in\mathcal A_0}
\left|
\widehat{\mathrm{CCT}}_\theta^{(k)}(\alpha)
-
\mathrm{CCT}_k(\alpha)
\right|
\\
&\quad\leq
\frac{
2\left(
C_2\mathcal R_2^{(k)}
+
C_3\mathcal R_3^{(k)}
\right)
+
\varepsilon
+
\log(N)/\beta
+
K_t h_t
}{
\mu
}
+
\tau_{\mathrm{root}}.
\end{aligned}
\label{eq:end-to-end-cct-control}
\end{equation}
\end{small}

\end{theorem}
Theorem~\ref{thm:end-to-end-cct} connects phase-wise residuals to
trajectory error and then to the reduced-model hard-CCT error. The
exact pre-fault representation and hard event chaining eliminate
separate initial-state, pre-fault, and state-interface defect terms.
The complete residual definitions, explicit constants, quantitative
small-error conditions, and proof are provided in the technical
supplement. The residual norms are
continuous quantities and are not automatically controlled by finite
collocation losses. Moreover, the theorem concerns the reduced
swing-model boundary, while agreement with the full-network DAE is
evaluated separately.

At a finite and uncensored learned root, assume that
\(M_\theta^{(k)}\) is locally \(C^1\) and
\(
\partial_{t_c}M_\theta^{(k)}
\left(
\alpha,t_\theta^{\circ,(k)}(\alpha)
\right)
\neq0.
\)
The implicit-function theorem then gives
\begin{equation}
\frac{
d t_\theta^{\circ,(k)}
}{
d\alpha
}
=
-
\frac{
\partial_\alpha M_\theta^{(k)}(\alpha,t_c)
}{
\partial_{t_c}M_\theta^{(k)}(\alpha,t_c)
}
\bigg|_{
t_c=t_\theta^{\circ,(k)}(\alpha)
}.
\label{eq:cct-sensitivity}
\end{equation}
Both derivatives are evaluated through the complete computational
graph at the root-refined learned boundary.

Extracting
\(\widehat{\mathrm{CCT}}_\theta^{(k)}\)
requires a clearing-time sweep followed by root refinement. For the
finite and uncensored extracted values, we train a post hoc scalar
readout
\(
\widetilde{\mathrm{CCT}}_\phi^{(k)}(\alpha)
\)
by minimizing
\begin{equation}
\min_\phi
\sum_{\alpha_i\in\mathcal I_{\mathrm{fin}}^{(k)}}
|
\widetilde{\mathrm{CCT}}_\phi^{(k)}(\alpha_i)
-
\widehat{\mathrm{CCT}}_\theta^{(k)}(\alpha_i)
|^2,
\label{eq:cct-distillation}
\end{equation}
where
\(\mathcal I_{\mathrm{fin}}^{(k)}\)
contains the finite and uncensored root-refined targets. This readout
amortizes the clearing-time sweep into a single forward query. It is
used only for deployment and does not define the stability boundary,
provide boundary sensitivities, or determine censoring.


\section{Experiments}
\label{sec:experiments}

We evaluate held-out hybrid-trajectory and CCT-boundary accuracy,
cross-system agreement with targeted full-network DAE references,
component contributions, deployment cost, and low-data multi-fault
representation sharing. The IEEE 30-bus restore-clearing benchmark
provides the primary matched comparison, while additional IEEE 9-,
14-, and 30-bus configurations test the trajectory-to-boundary
construction across fault mechanisms and clearing actions.

\subsection{Experimental Setup}
\label{subsec:experimental-setup}

We evaluate ES-PINN on the IEEE 9-, 14-, and 30-bus systems, with the steady-state bus, branch, and generator data taken from the corresponding MATPOWER cases~\cite{zimmerman2011matpower}. Unless
stated otherwise, \(\alpha\in[0,1]\), the fault is applied at
\(t_f=0.2\,\mathrm{s}\), trajectories are simulated until
\(T=2.0\,\mathrm{s}\), and supervised trajectory outputs are stored at
\(\Delta t=0.01\,\mathrm{s}\). The primary IEEE 30-bus benchmark uses
a midpoint transmission-line fault on line \((3,4)\) followed by
restoration of the pre-fault topology. The cross-system evaluation
additionally considers a mechanical-power perturbation with
restoration on IEEE 9, a midpoint line fault followed by line removal
on IEEE 14, and a load-bus shunt fault followed by feeder tripping on
IEEE 30. The mechanical perturbation uses
\(s_{\mathrm{mech}}=1.25\). Specific fault locations and clearing
actions are reported with the corresponding results.

For the primary IEEE 30-bus benchmark, all neural surrogates are
trained on a \(51\times51\) Cartesian grid over
\((\alpha,t_c)\). Evaluation uses a nested
\(101\times101\) grid after removing all \(2{,}601\) training-grid
pairs, leaving \(7{,}600\) held-out parameter pairs. Trajectory errors
and stability-map overlap are computed on this held-out set, while CCT
errors use the 50 interleaved severity values absent from the training
grid. Matched-grid reconstruction and mild severity extrapolation over
\(\alpha\in[1.02,1.20]\) are reported in the Appendix~\ref{app:additional-results}. Reduced swing-equation ODE trajectories provide the reference for the
controlled aggregate trajectory and stability-map comparisons.
Targeted full-network
DAE hard-CCT evaluations at selected severities provide an external check of the combined reduced-model and surrogate
discrepancy, but do not replace the reduced ODE in the matched
neural-surrogate comparison. All margin evaluations use
\(\Delta_{\max}=\pi\), \(\beta=50\),
\(\varepsilon=10^{-6}\), and \(h_t=1\,\mathrm{ms}\); the first
sign-changing cell is root-refined to
\(\tau_{\mathrm{root}}=10^{-5}\,\mathrm{s}\).
Complete numerical protocols are provided in the Appendix.

We compare ES-PINN with a monolithic Vanilla PINN~\cite{raissi2019physics}, an Enhanced XPINN
with soft event-interface constraints~\cite{jagtap2020extended}, and a physics-informed DeepONet~\cite{wang2021learning}.
All methods use matched trajectory supervision, physics-collocation
budgets, optimization protocols, and comparable or larger parameter
counts. We also evaluate three matched ablations that respectively
replace exact event chaining with a soft penalty, remove the prescribed
right-hand event derivatives, and remove the phase-wise residual losses.

We report rotor-angle and speed RMSE pooled over parameter pairs, time
samples, and generators. CCT accuracy is measured by the mean  absolute error relative to independently root-refined
reduced-ODE hard CCTs, using only finite and uncensored severity values.
Neural-surrogate comparisons, ablations, and multi-fault experiments
report mean and standard deviation over five seeds. Full architecture,
training, metric, distillation, and numerical-solver details are provided
in the Appendix.

\subsection{Trajectory and Stability-Boundary Accuracy}
\label{subsec:trajectory-boundary-results}

We first examine the primary IEEE 30-bus restore-clearing benchmark.
Figure~\ref{fig:case30-trajectory} shows a representative supercritical
off-grid case at
\(
\alpha=0.93
\)
and
\(
t_c=0.73\,\mathrm{s},
\)
for which the clearing time exceeds the reduced-ODE CCT. ES-PINN follows the reduced ODE and full-network DAE through fault
application, clearing, and the subsequent loss of synchronism. The corresponding rotor-angle trajectories for all generators are
shown in the Appendix~\ref{app:additional-results}.
Table~\ref{tab:case30_restore_results} provides the controlled
quantitative comparison over the complete held-out evaluation set. ES-PINN provides the best held-out trajectory and boundary accuracy.
Relative to Enhanced XPINN, the strongest trajectory baseline, it
reduces angle and speed RMSE by \(19\%\) and \(17\%\), respectively,
while reducing hard-CCT MAE by \(79\%\). The larger boundary
improvement indicates that modest aggregate trajectory errors can
still produce consequential CCT shifts. PI-DeepONet further
illustrates this distinction: its mean CCT error is comparatively
small despite substantially larger trajectory errors.
Additional reconstruction and extrapolation diagnostics are reported
in the Appendix~\ref{app:additional-results}.

\begin{figure}[t]
    \centering
    \includegraphics[width=\columnwidth]
    {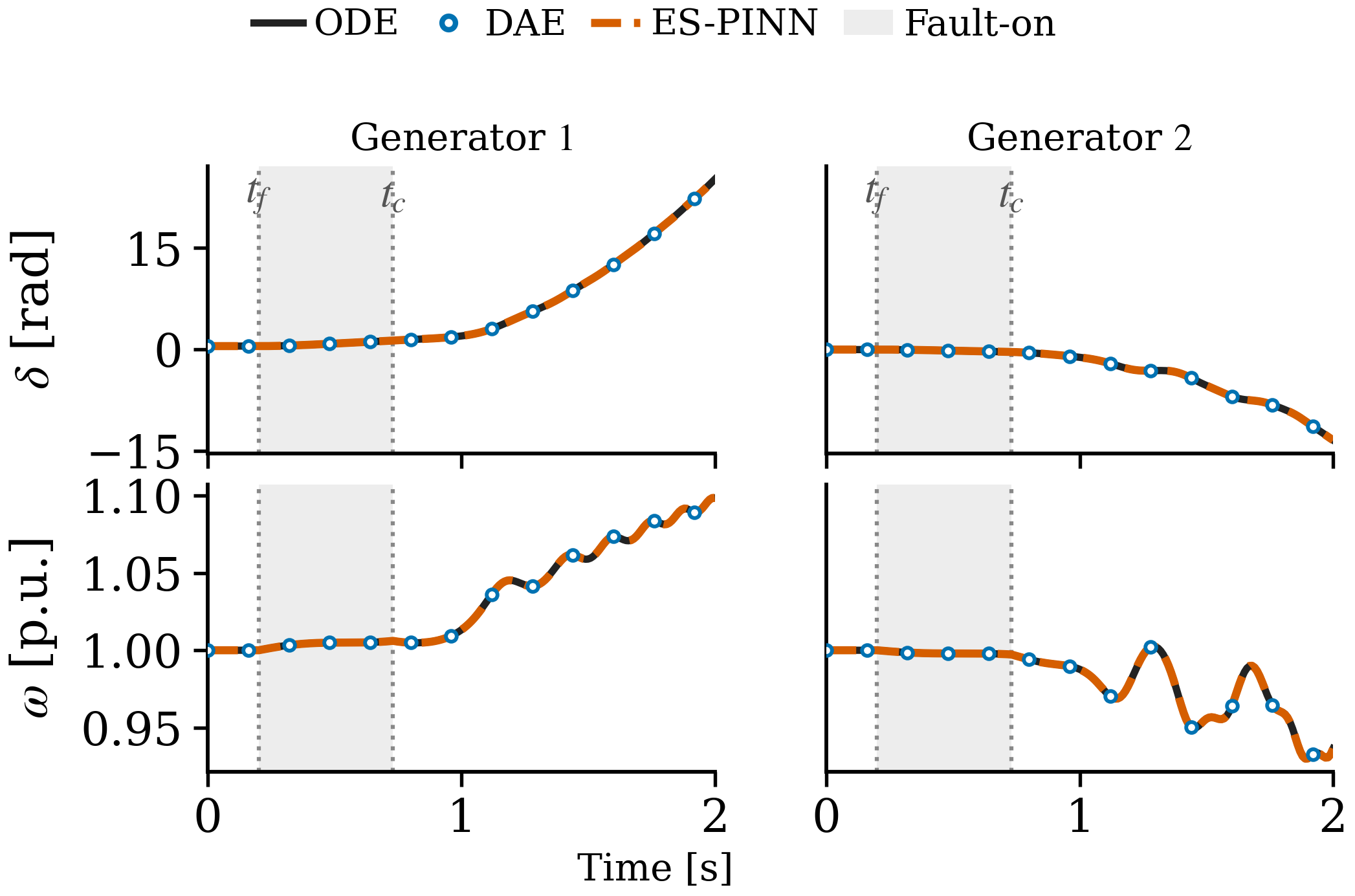}
    \caption{Representative  IEEE 30-bus trajectory for a
    midpoint fault on line \((3,4)\) followed by topology restoration.
    Rows show the rotor angle and speed of the ODE-selected critical
    generator pair. Shading denotes the fault-on phase.}
    \label{fig:case30-trajectory}
\end{figure}

\begin{table}[t]
\centering
\small
\setlength{\tabcolsep}{3.5pt}
\begin{tabular}{@{}lccc@{}}
\toprule
Method or variant
& $\mathrm{RMSE}_{\delta}$
& $\mathrm{RMSE}_{\omega}$
& CCT MAE \\
\midrule

\multicolumn{4}{@{}l}{\textit{Method comparison}} \\

Vanilla PINN
& $9.59\pm2.47$
& $13.2\pm4.1$
& $6.68\pm4.82$ \\

Enhanced XPINN
& $1.52\pm0.10$
& $2.33\pm0.06$
& $1.49\pm0.34$ \\

PI-DeepONet
& $26.2\pm11.6$
& $27.8\pm2.7$
& $1.87\pm0.95$ \\

ES-PINN
& $\mathbf{1.23\pm0.08}$
& $\mathbf{1.94\pm0.20}$
& $\mathbf{0.32\pm0.07}$ \\

\addlinespace[2pt]
\midrule
\multicolumn{4}{@{}l}{\textit{ES-PINN component ablations}} \\

State-only hard-chain
& $1.38\pm0.07$
& $2.14\pm0.70$
& $1.22\pm0.08$ \\

Soft-chain ES-PINN
& $8.97\pm1.72$
& $5.44\pm0.77$
& $1.10\pm0.37$ \\

w/o residual loss
& $4.48\pm0.42$
& $4.90\pm0.25$
& $4.06\pm0.73$ \\

\bottomrule
\end{tabular}
\caption{Held-out IEEE 30-bus restore-clearing accuracy and component
ablation results, reported as mean \(\pm\) standard deviation over five
seeds. The units of \(\mathrm{RMSE}_{\delta}\),
\(\mathrm{RMSE}_{\omega}\), and CCT MAE are
\(10^{-2}\,\mathrm{rad}\), \(10^{-4}\,\mathrm{p.u.}\), and
\(\mathrm{ms}\), respectively. Trajectory metrics use the \(7{,}600\)
off-grid parameter pairs, while CCT errors use independently
root-refined reduced-ODE hard-CCT references on the finite, uncensored
subset of the 50 unseen severity levels.}
\label{tab:case30_restore_results}
\end{table}

We next evaluate whether ES-PINN remains effective when trained
independently across different systems, fault mechanisms, and clearing actions:
an IEEE 9-bus mechanical perturbation with restoration, an IEEE
14-bus line fault with line removal, and an IEEE 30-bus shunt fault
with feeder tripping. Figure~\ref{fig:stability-cct} compares the
learned boundaries with reduced-ODE hard CCTs and targeted full-network DAE evaluations. Across all three configurations, ES-PINN tracks the reduced-ODE
first-crossing branches, including their finite and right-censored
regions. The targeted DAE comparison shows the same order of
full-network agreement for the reduced ODE and ES-PINN. Their
DAE-referenced MAEs are \(12.11\) and \(13.10\,\mathrm{ms}\) on IEEE
9, and \(4.51\) and \(3.94\,\mathrm{ms}\) on IEEE 30, respectively.
On IEEE 14, the learned and DAE boundaries nearly coincide within the
resolution of the targeted validation. This configuration-specific
agreement should not be interpreted as systematic correction of the
reduced model from which ES-PINN is learned. A cross-system \(\beta\)-sweep and a Case-30 temporal-sampling study, reported in the Appendix~\ref{app:numerical-reliability}, show that smooth-margin and sampling errors are substantially smaller than the learned end-to-end boundary error; for Case 30 at the default \(\beta=50\), the smooth-to-hard CCT discrepancy is only \(0.006\,\mathrm{ms}\).
At finite regular roots, the implicit sensitivities agree with centered
finite differences of the same learned CCT branches. Detailed
consistency statistics and endpoint diagnostics are reported in the
Appendix~\ref{app:numerical-reliability}.

\begin{figure}[t]
    \centering
    \includegraphics[width=\linewidth]
    {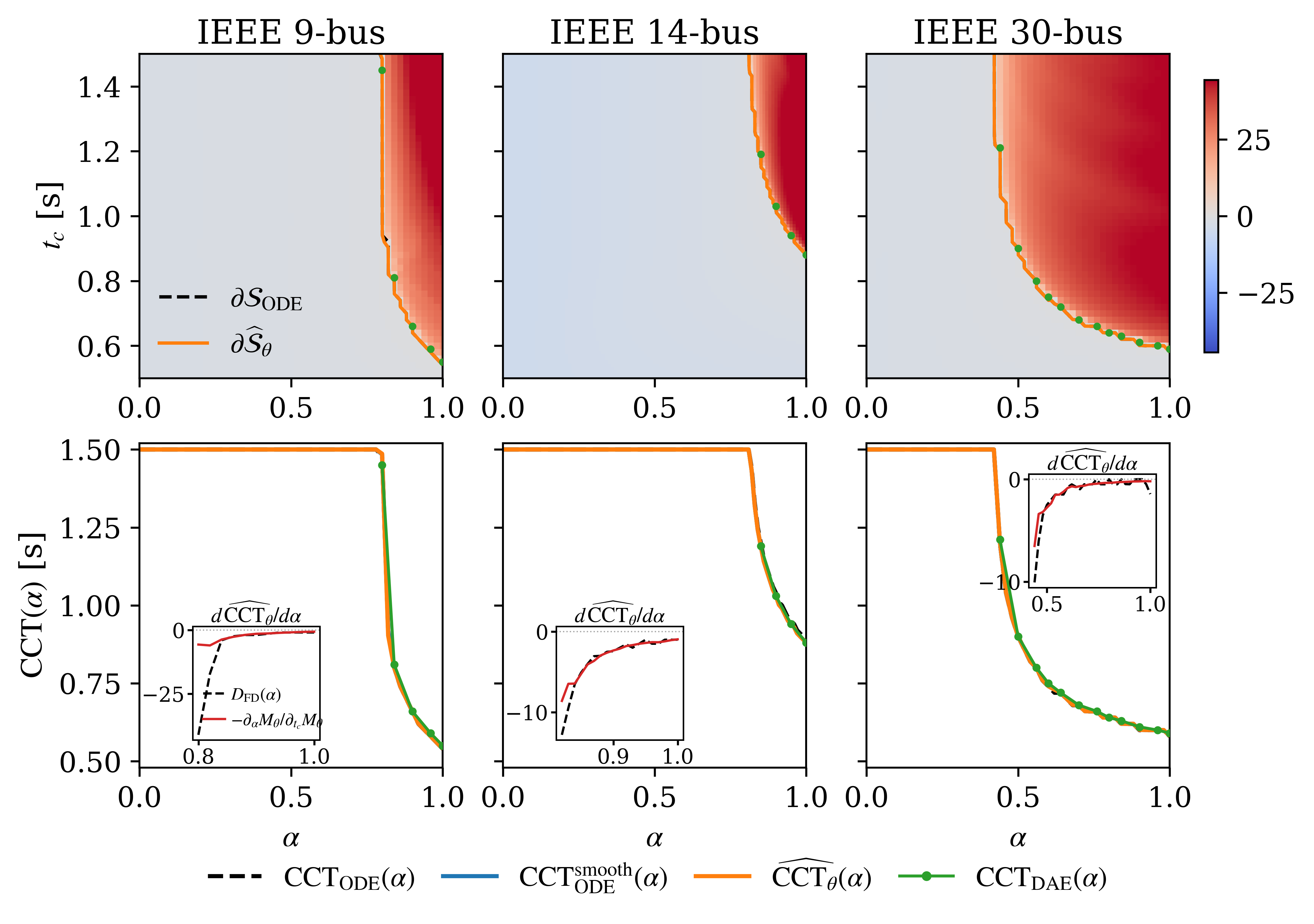}
    \caption{Cross-system trajectory-induced margins and CCT
    boundaries. Top: learned margins with reduced-ODE hard boundaries,
    root-refined ES-PINN boundaries, and available full-network DAE
    points. Bottom: the corresponding reduced-ODE, sampled-smooth,
    ES-PINN, and finite DAE CCT evaluations. Insets compare implicit
    learned-boundary sensitivities with centered finite differences.}
    \label{fig:stability-cct}
\end{figure}

Table~\ref{tab:case30_restore_results} isolates the contributions of
exact event-state chaining, prescribed right-hand event derivatives,
and phase-wise residual regularization.
State-only hard chaining produces only modest aggregate trajectory degradation
but increases CCT MAE by \(3.8\times\), indicating that the physical
right-hand event derivative is particularly important for boundary
accuracy. Replacing exact chaining with a matched soft interface
penalty increases angle RMSE by \(7.3\times\) and CCT MAE by
\(3.4\times\), directly supporting hard event-state chaining. Removing
the residual loss yields the largest CCT degradation, showing that
local Taylor structure and trajectory supervision do not replace
phase-wise residual regularization.


Finally, we evaluate whether the post-hoc scalar head can reconstruct
the root-refined learned CCT boundary without performing a
clearing-time sweep. For each finite and uncensored severity
\(\alpha\), we define the signed reconstruction discrepancy as
\[
\Delta t_c(\alpha)
=
10^3
\left[
\widetilde{\mathrm{CCT}}_{\phi}(\alpha)
-
\widehat{\mathrm{CCT}}_{\theta}(\alpha)
\right]
\ \mathrm{ms},
\]
where \(\widetilde{\mathrm{CCT}}_{\phi}\) is the distilled-head
prediction and \(\widehat{\mathrm{CCT}}_{\theta}\) is the root-refined
ES-PINN boundary. We report the mean absolute value of
\(\Delta t_c\), which measures reconstruction of the learned boundary
rather than independent physical CCT accuracy.
Figure~\ref{fig:cct-distillation} compares the distilled readouts with
the corresponding root-refined learned boundaries. On the finite
target grids, the distilled heads attain mean absolute reconstruction
errors of \(2.1\), \(2.1\), and \(4.3\,\mathrm{ms}\) for the IEEE
9-, 14-, and 30-bus configurations, respectively.

\begin{figure}[t]
    \centering
    \includegraphics[width=\columnwidth]
    {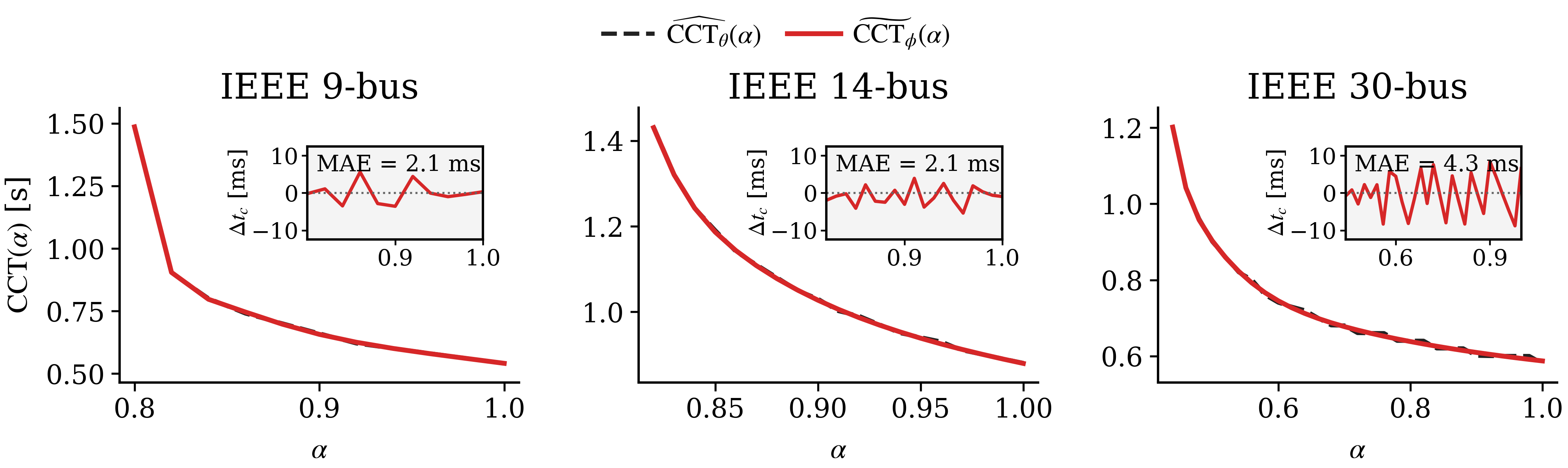}
    \caption{Post hoc scalar reconstruction of the learned CCT
    boundaries in Fig.~\ref{fig:stability-cct}. Dashed curves denote
    root-refined learned boundaries, red curves denote distilled
    readouts, and the insets show the signed discrepancy
    \(\Delta t_c\).}
    \label{fig:cct-distillation}
\end{figure}









\begin{table}[t]
\centering
\small
\setlength{\tabcolsep}{3.5pt}
\begin{tabular}{@{}llcc@{}}
\toprule
Evaluator & Device & CCT runtime  & Speedup vs.\ ODE \\
\midrule
\multicolumn{4}{@{}l}{\textit{Boundary extraction}}\\

Full-network DAE & CPU
& $202{,}187.8$
& -- \\

Reduced ODE & CPU
& $1{,}029.8$
& $1.00\times$ \\

ES-PINN & CPU
& $542.2$
& $1.90\times$ \\

ES-PINN & CUDA
& $\mathbf{80.4}$
& $\mathbf{12.81\times}$ \\

\addlinespace[2pt]
\multicolumn{4}{@{}l}{\textit{Direct deployment readout}}\\

Distilled CCT head & CPU
& $\mathbf{0.064}$
& $\mathbf{1.60\times10^{4}}$ \\

\bottomrule
\end{tabular}%
\caption{Median IEEE 30-bus CCT-query runtime over the 50 unseen
severity levels. Runtime is reported in milliseconds. ODE and ES-PINN evaluate 101 clearing-time candidates
followed by root refinement, while the DAE uses adaptive hard-margin
extraction. }
\label{tab:case30_cct_runtime}
\end{table}
Table~\ref{tab:case30_cct_runtime} compares boundary-extraction and
direct-readout costs. Timings exclude model loading and offline training, use
one CPU core or one CUDA device, and are measured after warm-up. The
distilled-head latency uses 2,000 batch-size-one forwards per severity.
Speedups are relative to reduced-ODE extraction. For the same clearing-time sweep and root-refinement
procedure, ES-PINN achieves a \(1.90\times\) CPU speedup over the
reduced ODE, increasing to \(12.81\times\) on CUDA. The distilled head
instead performs an amortized scalar query and therefore does not
execute the same boundary-extraction procedure; its median latency is
\(0.064\,\mathrm{ms}\). 


\subsection{Multi-Fault Representation Sharing}
\label{subsec:multi-fault-results}
We compare two independently trained fault-specific ES-PINNs
(Separate) with one hard-routed model that shares representations
across a load-bus shunt fault and a transmission-line fault (MoE). The
two families are a shunt fault at bus \(7\), cleared by removing line
\((4,6)\), and a midpoint fault on line \((3,4)\), cleared by removing
the faulted line. At each data fraction and seed, both methods receive
identical subsets of complete \((\alpha,t_c)\) scenarios. Pooled
trajectory metrics use both fault families, whereas CCT error is
evaluated only on the boundary-bearing line fault because the load-bus
configuration remains right-censored throughout the queried range.
Complete grid and architecture details are provided in the Appendix~\ref{app:reproducibility}. 

Table~\ref{tab:moe_low_data} shows that representation sharing improves
pooled trajectory accuracy at every data fraction, and the
\(10\%\)-data MoE already outperforms the full-data Separate models on
both state metrics. Its boundary benefit
is data-dependent: MoE attains lower line-fault CCT error at the
\(0.1\) and \(0.2\) data fractions, while Separate regains a modest
advantage at \(0.5\) and full data. Thus, sharing improves low-data
robustness but does not replace fault-specific specialization when
each family is sufficiently sampled. MoE uses \(50.8\%\) fewer
parameters and consolidates both fault families into one checkpoint,
although it requires more total training time. The details are shown in the Appendix.





\begin{table}[t]
\centering
\small
\setlength{\tabcolsep}{2.2pt}
\begin{tabular}{@{}clccc@{}}
\toprule
Data frac. & Method
& \(\mathrm{RMSE}_{\delta}\)
& \(\mathrm{RMSE}_{\omega}\)
& CCT MAE \\
\midrule

0.1 & Separate
& $0.21\pm0.16$
& $10.21\pm5.13$
& $7.19\pm5.00$ \\
& MoE
& $\boldsymbol{0.15\pm0.02}$
& $\boldsymbol{4.01\pm1.25}$
& $\boldsymbol{3.92\pm1.85}$ \\
\addlinespace

0.2 & Separate
& $0.19\pm0.11$
& $9.59\pm4.88$
& $4.39\pm1.95$ \\
& MoE
& $\boldsymbol{0.16\pm0.01}$
& $\boldsymbol{3.97\pm1.63}$
& $\boldsymbol{2.36\pm0.68}$ \\
\addlinespace

0.5 & Separate
& $0.18\pm0.09$
& $8.34\pm2.21$
& $\boldsymbol{1.38\pm0.49}$ \\
& MoE
& $\boldsymbol{0.16\pm0.02}$
& $\boldsymbol{3.79\pm1.62}$
& $1.79\pm1.10$ \\
\addlinespace

1.0 & Separate
& $0.18\pm0.06$
& $8.21\pm2.36$
& $\boldsymbol{0.76\pm0.24}$ \\
& MoE
& $\boldsymbol{0.15\pm0.01}$
& $\boldsymbol{4.34\pm1.73}$
& $0.97\pm0.26$ \\

\bottomrule
\end{tabular}
\caption{Low-data multi-fault interpolation on nested dense-grid
parameter pairs absent from the native grids. Values are reported as
mean \(\pm\) standard deviation over five seeds. The units of
\(\mathrm{RMSE}_{\delta}\), \(\mathrm{RMSE}_{\omega}\), and CCT MAE
are \(\mathrm{rad}\), \(10^{-4}\,\mathrm{p.u.}\), and
\(\mathrm{ms}\), respectively. Pooled trajectory metrics include both
fault families, while hard-CCT MAE is computed on the finite,
uncensored unseen severity levels of the boundary-bearing line fault.
Lower is better for all metrics; bold indicates the better result at
each data fraction.}
\label{tab:moe_low_data}
\end{table}

\section{Conclusion}
The results show that stability-boundary accuracy depends not only on
aggregate trajectory fit, but also on how switching events are
represented. The ablations identify complementary roles for exact
state chaining, the physical right-hand event derivative, and
phase-wise residual regularization, while the cross-system and targeted
DAE evaluations support the findings beyond the primary benchmark. The
error analysis further connects these trajectory-level approximations
to CCT accuracy. This highlights the importance of preserving event-local
structure near critical boundaries, where small interface errors can
translate into appreciable CCT shifts. The current scope assumes known phase-wise reduced dynamics, known
event rules, and continuous generator states. Future work will extend
the approach to higher-fidelity DAEs, uncertain and multi-stage events,
and larger contingency families, with the longer-term goal of
uncertainty-aware hybrid surrogates that retain reliable operational
boundaries.



\appendix

\section{Architectures and Reproducibility Details}
\label{app:reproducibility}

This section records the neural architectures and numerical settings
used in the primary IEEE 30-bus restore-clearing benchmark. All four
surrogates receive the same supervised trajectories, parameter grid,
phase-wise physics-collocation budgets, optimization schedule, and
five random seeds. The three component variants retain the same
phase-network dimensions, training data, and optimization protocol
while changing only the specified event or loss mechanism.

The system contains six generators, so each surrogate predicts
\(2n_g=12\) state variables. Table~\ref{tab:case30-architectures}
summarizes the architectures and trainable parameter counts. The
notation \([m]\times d\) denotes \(d\) hidden layers of width \(m\).

\begin{table*}[t]
\centering
\caption{Neural architectures used in the IEEE 30-bus
restore-clearing benchmark. ES-PINN and Enhanced XPINN use identical
phase-network dimensions.}
\label{tab:case30-architectures}
\small
\setlength{\tabcolsep}{4pt}
\begin{tabular}{@{}p{2.2cm}p{5.5cm}p{2.4cm}p{4.3cm}r@{}}
\toprule
Method
& Neural architecture
& Activation
& Event treatment
& Parameters \\
\midrule

ES-PINN
&
D1: exact equilibrium;
D2: \(2\!-\![64]\times5\!-\!12\);
D3: \(3\!-\![128]\times7\!-\!12\)
&
Sine, \(w_0=5\)
&
Hard event-state chaining with local Taylor-gated corrections
&
\(118{,}744\)
\\

Enhanced XPINN
&
D1: exact equilibrium;
D2: \(2\!-\![64]\times5\!-\!12\);
D3: \(3\!-\![128]\times7\!-\!12\)
&
Sine, \(w_0=5\)
&
Independent phase networks with soft interfaces at \(t_f\) and \(t_c\)
&
\(118{,}744\)
\\

Vanilla PINN
&
Global MLP:
\(3\!-\![138]\times8\!-\!12\)
&
Tanh
&
Global representation with a soft initial-state penalty
&
\(136{,}494\)
\\

PI-DeepONet
&
Branch:
\(2\!-\![129]\times4\!-\!129\);
trunk:
\(1\!-\![129]\times4\!-\!129\);
linear \(129\!-\!12\) readout
&
Tanh
&
Global operator representation with a soft initial-state penalty
&
\(136{,}365\)
\\

\bottomrule
\end{tabular}
\end{table*}

Replacing the learned pre-fault component by the exact equilibrium
reduces the ES-PINN and Enhanced XPINN parameter counts. Consequently,
the Vanilla PINN and PI-DeepONet contain approximately \(15\%\) more
trainable parameters than ES-PINN. Neither global baseline therefore
receives less neural capacity than the proposed model. Soft-chain
ES-PINN, State-only hard-chain, and \emph{w/o residual} all use the
same \(118{,}744\)-parameter phase networks as the complete model;
the interface penalty and gate changes introduce no additional
trainable parameters.

The severity and clearing-time coordinates are normalized over their
training ranges. Loss-weight calibration is performed independently
for every method and random seed. The same five-seed protocol is used
for the main comparison and the three component variants, while the
full-network DAE reference and numerical-convergence diagnostics are
deterministic.

\section{Fault Models and Clearing Operators}
\label{app:fault-models}

This section specifies the contingency constructions used to generate
the phase-wise network and mechanical operators. The same fault and
clearing definitions are used by the reduced swing-equation reference,
the physics-informed surrogates, and the full-network DAE simulations.
For the reduced model, constant-impedance loads are first absorbed into
the bus-admittance matrix, after which the network is augmented with
the generator transient-reactance branches. The resulting matrix is
partitioned as
\[
\widetilde{Y}
=
\begin{bmatrix}
Y_{\mathrm{bb}} & Y_{\mathrm{bg}}\\
Y_{\mathrm{gb}} & Y_{\mathrm{gg}}
\end{bmatrix},
\]
where the subscripts \(\mathrm{b}\) and \(\mathrm{g}\) denote
algebraic bus and generator-internal variables, respectively. The
Kron-reduction map is
\[
\mathcal{K}(Y_{\mathrm{bus}})
=
Y_{\mathrm{gg}}
-
Y_{\mathrm{gb}}
Y_{\mathrm{bb}}^{-1}
Y_{\mathrm{bg}}.
\]
The full-network DAE retains the algebraic bus variables and therefore
uses the corresponding unreduced network matrices.

For contingency member \(k\), the phase-wise reduced electrical
operators are
\[
Y_{\mathrm{red}}^{\mathrm{flt},k}(\alpha)
=
\mathcal{K}\!\left(
Y_{\mathrm{bus}}^{\mathrm{flt},k}(\alpha)
\right),
\qquad
Y_{\mathrm{red}}^{\mathrm{pst},k}
=
\mathcal{K}\!\left(
Y_{\mathrm{bus}}^{\mathrm{pst},k}
\right).
\]
Together with the prescribed mechanical inputs
\(P_m^{\mathrm{flt},k}(\alpha)\) and
\(P_m^{\mathrm{pst},k}\), these matrices define the fault-on and
post-clearing vector fields.

For a generator-bus or load-bus shunt fault at bus \(b_k\), the
fault-on network is
\[
Y_{\mathrm{bus}}^{\mathrm{flt},k}(\alpha)
=
Y_{\mathrm{bus}}^{\mathrm{pre},k}
+
\alpha y_{b_k}^{(k)}
\mathbf e_{b_k}\mathbf e_{b_k}^{\mathsf T},
\]
where \(\mathbf e_{b_k}\) is the bus-selection vector and
\(y_{b_k}^{(k)}\) is the reference shunt admittance assigned to the
contingency. Hence \(\alpha=0\) recovers the pre-fault network, while
increasing \(\alpha\) strengthens the shunt-to-ground perturbation.

For a mechanical-power perturbation at generator \(g_k\), the
fault-on electrical network remains unchanged, and
\[
P_{m,g_k}^{\mathrm{flt},k}(\alpha)
=
\left(
1+s_{\mathrm{mech}}^{(k)}\alpha
\right)
P_{m,g_k}^{\mathrm{pre},k},
\]
with the remaining mechanical inputs unchanged. The post-clearing
vector \(P_m^{\mathrm{pst},k}\) is specified by the contingency and
need not equal the pre-fault vector in the general formulation. For
the restore-clearing mechanical perturbations used in the experiments,
the selected mechanical input is reset at clearing, giving
\[
P_m^{\mathrm{pst},k}
=
P_m^{\mathrm{pre},k}.
\]
Thus, the physical interpretation of \(\alpha\) is
configuration-specific: it scales an electrical shunt for bus faults
and a selected mechanical input for mechanical perturbations.

For a fault on transmission line \((i_k,j_k)\), let
\[
y_{i_kj_k}
=
\frac{1}{
r_{i_kj_k}+\mathrm{i}x_{i_kj_k}
}
\]
denote the line-series admittance, and let
\(\lambda_k\in(0,1)\) denote the fault position measured from
\(i_k\). The original series branch is replaced by two pure-series
segments connected through a virtual fault node \(f_k\):
\[
y_{i_k f_k}
=
\frac{y_{i_kj_k}}{\lambda_k},
\qquad
y_{f_k j_k}
=
\frac{y_{i_kj_k}}{1-\lambda_k}.
\]
The original terminal shunts of the line \(\pi\)-model remain at
\(i_k\) and \(j_k\), and the virtual node receives only the
severity-dependent fault shunt
\[
y_{f_k}^{\mathrm{sh}}(\alpha)
=
\alpha y_{i_kj_k}.
\]
This is a finite-admittance severity parameterization: \(\alpha=0\)
recovers the unfaulted line, and increasing \(\alpha\) approaches a
stronger shunt fault. A strict bolted-fault limit would require an
unbounded shunt admittance and is not implied by \(\alpha=1\).

After constructing the augmented bus matrix, partition it with respect
to the original buses and the scalar virtual-node variable:
\[
Y_{\mathrm{aug}}^{\mathrm{flt},k}(\alpha)
=
\begin{bmatrix}
Y_{\mathrm{bb}}^{\mathrm{aug},k}(\alpha)
&
\mathbf y_{\mathrm{b}f}^{(k)}
\\
\mathbf y_{f\mathrm{b}}^{(k)}
&
y_{ff}^{(k)}(\alpha)
\end{bmatrix}.
\]
Eliminating \(f_k\) gives
\[
Y_{\mathrm{bus}}^{\mathrm{flt},k}(\alpha)
=
Y_{\mathrm{bb}}^{\mathrm{aug},k}(\alpha)
-
\mathbf y_{\mathrm{b}f}^{(k)}
\left(
y_{ff}^{(k)}(\alpha)
\right)^{-1}
\mathbf y_{f\mathrm{b}}^{(k)}.
\]
Because only the original series contribution is replaced by the two
segments while the terminal shunts are retained, this construction
recovers the original pre-fault \(\pi\)-model exactly at
\(\alpha=0\).

Clearing removes the active fault shunt and selects the post-clearing
network according to the prescribed protection action:
\[
Y_{\mathrm{bus}}^{\mathrm{pst},k}
=
\begin{cases}
Y_{\mathrm{bus}}^{\mathrm{pre},k},
& \text{topology restoration},\\[2mm]
\mathcal{R}_{\ell_k}
\!\left(
Y_{\mathrm{bus}}^{\mathrm{pre},k}
\right),
& \text{removal of line }\ell_k,\\[2mm]
\mathcal{R}_{\mathcal{L}(b_k)}
\!\left(
Y_{\mathrm{bus}}^{\mathrm{pre},k}
\right),
& \text{incident-feeder tripping at }b_k.
\end{cases}
\]
Here, \(\mathcal{R}_{\ell}\) removes the complete \(\pi\)-model stamp
of line \(\ell\), including its terminal shunts, and
\(\mathcal{L}(b_k)\) denotes the set of removable non-transformer
branches incident to the faulted bus. Transformer and phase-shifting
branches are excluded from the removable-line set.

Consequently, \(\alpha\) controls the fault-on severity, whereas the
contingency index \(k\) specifies the fault location, fault mechanism,
and prescribed clearing action. The clearing time \(t_c\) changes only
the switching instant and does not alter either the fault-on operator
or the prescribed post-clearing operator.

\section{Formal Assumptions and Proof of
the
Uniform Residual-to-Hard-CCT Control Theorem}
\label{app:end-to-end-cct-proof}

We first state the formal assumptions suppressed from the compact
theorem statement and then prove the result in three steps:
phase-wise residual control of the hybrid trajectory, control of the
reduced-model hard margin by the learned sampled-smooth margin, and
preservation and refinement of the initial hard-CCT crossing.

\paragraph{Formal assumptions and notation.}

Fix a contingency member \(k\) and a compact severity interval
\(\mathcal A_0\subseteq\mathcal A_k\) on which the reduced-model hard
CCT is finite and uncensored. For
\(\alpha\in\mathcal A_0\), define
\(
t_c^\star(\alpha)
=
\mathrm{CCT}_k(\alpha).
\)
For each \(t_c\in\mathcal T=[t_{c,\min},t_{c,\max}]\), let
\(
I_2(t_c)=[t_f,t_c],
\;
I_3(t_c)=[t_c,T].
\)

Let \(r_q^{(k)}\) denote the vector phase residual. Consistently with the
non-anticipative fault-on representation, define the uniform
continuous residual norms
\begin{equation}
\begin{aligned}
\mathcal R_2^{(k)}
&=
\sup_{\substack{
\alpha\in\mathcal A_0\\
t_c\in\mathcal T
}}
\left\|
r_2^{(k)}(\cdot;\alpha)
\right\|_{L^1(I_2(t_c))},
\\
\mathcal R_3^{(k)}
&=
\sup_{\substack{
\alpha\in\mathcal A_0\\
t_c\in\mathcal T
}}
\left\|
r_3^{(k)}(\cdot;\alpha,t_c)
\right\|_{L^1(I_3(t_c))}.
\end{aligned}
\label{eq:uniform-phase-residuals}
\end{equation}
These are continuous residual norms rather than empirical collocation
losses.

Assume that \(f_2^{(k)}\) and \(f_3^{(k)}\) are uniformly Lipschitz,
with constants \(\Lambda_2\) and \(\Lambda_3\), on a common bounded
tube containing the exact and learned trajectories. Define the
maximum phase durations
\(
\overline{\tau}_2
=
t_{c,\max}-t_f,
\;
\overline{\tau}_3
=
T-t_{c,\min},
\)
and
\begin{equation}
C_2
=
\exp\left(
\Lambda_2\overline{\tau}_2
+
\Lambda_3\overline{\tau}_3
\right),
\qquad
C_3
=
\exp\left(
\Lambda_3\overline{\tau}_3
\right).
\label{eq:appendix-trajectory-constants}
\end{equation}
The learned trajectory uses the exact pre-fault equilibrium and exact
state chaining at both event interfaces.

Let
\(
\mathcal T_h
=
\{t_\ell\}_{\ell=1}^{N_t},
\;
0=t_1<\cdots<t_{N_t}=T,
\)
be the trajectory-time grid, with maximum adjacent spacing
\(h_t\). Assume that every pairwise reduced-ODE rotor-angle separation
is uniformly \(K_t\)-Lipschitz in time over
\(\mathcal A_0\times\mathcal T\).

Let
\(
\mathcal T_{h_c}
=
\{t_{c,0},\ldots,t_{c,J}\},
\;
t_{c,0}=t_{c,\min},
\;
t_{c,J}=t_{c,\max},
\)
denote the ordered clearing-time grid used to identify the first
sign-changing cell, and define
\(
h_c
=
\max_{0\leq j<J}
\left(
t_{c,j+1}-t_{c,j}
\right).
\)

The continuous neural networks and event gates, together with exact
state chaining, make
\(M_\theta^{(k)}(\alpha,\cdot)\) continuous in \(t_c\).
In particular, when a trajectory-time sample changes from the
fault-on to the post-clearing component at \(t=t_c\), the two
representations agree exactly at the interface. The learned margin is
locally \(C^1\) away from trajectory-time samples that coincide with
an event.

Assume that the initial hard-CCT crossing is uniformly interior,
isolated from earlier crossings, and transverse. Specifically, there
exist constants
\(\rho,\gamma,\mu>0\), independent of
\(\alpha\in\mathcal A_0\), such that
\begin{equation}
\left[
t_c^\star(\alpha)-\rho,\,
t_c^\star(\alpha)+\rho
\right]
\subseteq
\mathcal T,
\qquad
\alpha\in\mathcal A_0,
\label{eq:cct-uniform-interiority}
\end{equation}
and
\begin{equation}
m^{(k)}(\alpha,t_c)
\leq
-\gamma,
\qquad
t_c\in
[t_{c,\min},t_c^\star(\alpha)-\rho].
\label{eq:cct-prebranch-separation}
\end{equation}
Within the local branch neighborhood, assume
\begin{equation}
\operatorname{sgn}
\left(
t_c-t_c^\star(\alpha)
\right)
m^{(k)}(\alpha,t_c)
\geq
\mu
\left|
t_c-t_c^\star(\alpha)
\right|,
\label{eq:cct-branch-regularity}
\end{equation}
for $\left|
t_c-t_c^\star(\alpha)
\right|
\leq\rho.$
Define the combined trajectory and margin error
\begin{equation}
\Xi^{(k)}
=
2
\left(
C_2\mathcal R_2^{(k)}
+
C_3\mathcal R_3^{(k)}
\right)
+
\varepsilon
+
\frac{\log N}{\beta}
+
K_t h_t,
\label{eq:total-margin-error}
\end{equation}
where
\(N=N_t\binom{n_g}{2}\), and assume
\begin{equation}
\Xi^{(k)}<\gamma,
\qquad
\frac{2\Xi^{(k)}}{\mu}+h_c<\rho.
\label{eq:cct-small-error-condition}
\end{equation}

Throughout the proof, we use a state norm whose restriction to the
rotor-angle components dominates the componentwise
\(\ell_\infty\) norm.

\begin{lemma}[Exact-chain residual-to-trajectory control]
\label{lem:hybrid-trajectory-control}
Under the assumptions above,
\begin{small}
    \begin{equation}
\sup_{\substack{
\alpha\in\mathcal A_0,\,
t_c\in\mathcal T\\
t\in[0,T]
}}
\left\|
\widehat{x}^{(k)}_\theta(t;\alpha,t_c)
-
x^{(k)}(t;\alpha,t_c)
\right\|
\leq
C_2\mathcal R_2^{(k)}
+
C_3\mathcal R_3^{(k)}.
\label{eq:appendix-trajectory-control}
\end{equation}
\end{small}
\end{lemma}

\begin{proof}
Fix arbitrary
\(\alpha\in\mathcal A_0\) and \(t_c\in\mathcal T\), and suppress
these arguments when no ambiguity arises.

On the fault-on interval \(I_2(t_c)\), define
\[
e_2(t)
=
x^{(k,2)}(t;\alpha)
-
\widehat{x}^{(k,2)}_\theta(t;\alpha).
\]
The exact pre-fault representation and hard fault-on initialization
give
\(
e_2(t_f)=0.
\)
By the exact dynamics and the phase-2 residual definition,
\[
\partial_t e_2
=
f_2^{(k)}
\left(
x^{(k,2)};\alpha
\right)
-
f_2^{(k)}
\left(
\widehat{x}^{(k,2)}_\theta;\alpha
\right)
-
r_2^{(k)}(\cdot;\alpha).
\]
The \(\Lambda_2\)-Lipschitz continuity of \(f_2^{(k)}\) therefore
implies, almost everywhere on \(I_2(t_c)\),
\[
\frac{d}{dt}\|e_2(t)\|
\leq
\Lambda_2\|e_2(t)\|
+
\|r_2^{(k)}(t;\alpha)\|.
\]
Gronwall's inequality gives
\begin{align}
\|e_2(t)\|
&\leq
\int_{t_f}^{t}
e^{\Lambda_2(t-s)}
\|r_2^{(k)}(s;\alpha)\|\,ds
\nonumber\\
&\leq
e^{\Lambda_2\overline{\tau}_2}
\mathcal R_2^{(k)},
\qquad
t\in I_2(t_c).
\label{eq:d2-gronwall}
\end{align}

On the post-clearing interval \(I_3(t_c)\), define
\[
e_3(t)
=
x^{(k,3)}(t;\alpha,t_c)
-
\widehat{x}^{(k,3)}_\theta(t;\alpha,t_c).
\]
Continuity of the exact trajectory and hard chaining of the learned
trajectory give
\[
e_3(t_c)
=
x^{(k,2)}(t_c;\alpha)
-
\widehat{x}^{(k,2)}_\theta(t_c;\alpha)
=
e_2(t_c).
\]
Thus, clearing introduces no independent state-interface defect.

The \(\Lambda_3\)-Lipschitz continuity of \(f_3^{(k)}\) similarly
gives
\[
\frac{d}{dt}\|e_3(t)\|
\leq
\Lambda_3\|e_3(t)\|
+
\|r_3^{(k)}(t;\alpha,t_c)\|
\]
almost everywhere on \(I_3(t_c)\). A second application of
Gronwall's inequality yields
\begin{align}
\|e_3(t)\|
&\leq
e^{\Lambda_3(t-t_c)}
\left[
\|e_3(t_c)\|
+
\left\|
r_3^{(k)}(\cdot;\alpha,t_c)
\right\|_{L^1(I_3(t_c))}
\right]
\nonumber\\
&\leq
e^{\Lambda_3\overline{\tau}_3}
\left[
e^{\Lambda_2\overline{\tau}_2}
\mathcal R_2^{(k)}
+
\mathcal R_3^{(k)}
\right]
\nonumber\\
&=
C_2\mathcal R_2^{(k)}
+
C_3\mathcal R_3^{(k)}.
\label{eq:d3-gronwall}
\end{align}

The pre-fault error is identically zero because
\[
x^{(k,1)}(t)
=
\widehat{x}^{(k,1)}_\theta(t)
=
x_0.
\]
Equations~\eqref{eq:d2-gronwall} and
\eqref{eq:d3-gronwall} therefore control the complete hybrid
trajectory. All constants are uniform over
\(\alpha\in\mathcal A_0\) and \(t_c\in\mathcal T\), which proves
Eq.~\eqref{eq:appendix-trajectory-control}.
\end{proof}

\begin{lemma}[Learned sampled-smooth margin to reduced-model hard margin]
\label{lem:trajectory-margin-control}
Let
\(
\mathcal E_x^{(k)}
=
C_2\mathcal R_2^{(k)}
+
C_3\mathcal R_3^{(k)}.
\)
Then
\begin{equation}
\begin{aligned}
    \sup_{\substack{
\alpha\in\mathcal A_0\\
t_c\in\mathcal T
}}
&\left|
M_\theta^{(k)}(\alpha,t_c)
-
m^{(k)}(\alpha,t_c)
\right|
\\ &\leq
2\mathcal E_x^{(k)}
+
\varepsilon
+
\frac{\log N}{\beta}
+
K_t h_t
=
\Xi^{(k)}.
\end{aligned}
\label{eq:learned-to-hard-margin-control}
\end{equation}
\end{lemma}

\begin{proof}
For a rotor-angle trajectory \(\delta\), define its continuous-time
and sampled hard separations by
\[
H[\delta]
=
\max_{\substack{
t\in[0,T]\\
i<j
}}
\left|
\delta_i(t)-\delta_j(t)
\right|,
\;
H_h[\delta]
=
\max_{\substack{
t_\ell\in\mathcal T_h\\
i<j
}}
\left|
\delta_i(t_\ell)-\delta_j(t_\ell)
\right|.
\]
The reduced-model hard margin is
\[
m^{(k)}(\alpha,t_c)
=
H\left[
\delta^{(k)}(\cdot;\alpha,t_c)
\right]
-
\Delta_{\max}.
\]

Because the trajectory-time grid spans \([0,T]\) with maximum spacing
\(h_t\), every continuous-time maximizer lies within distance at most
\(h_t\) of a sampled time. The uniform temporal Lipschitz condition
therefore gives
\begin{equation}
0
\leq
H[\delta]-H_h[\delta]
\leq
K_t h_t.
\label{eq:time-sampling-margin-bound}
\end{equation}

For every \(z\in\mathbb R\),
\[
|z|
\leq
\sqrt{z^2+\varepsilon^2}
\leq
|z|+\varepsilon.
\]
Thus, if
\[
V_h[\delta]
=
\max_{\substack{
t_\ell\in\mathcal T_h\\
i<j
}}
\sqrt{
\left(
\delta_i(t_\ell)-\delta_j(t_\ell)
\right)^2
+
\varepsilon^2
},
\]
then
\begin{equation}
H_h[\delta]
\leq
V_h[\delta]
\leq
H_h[\delta]+\varepsilon.
\label{eq:smooth-absolute-bound}
\end{equation}

The log-sum-exp operator satisfies
\[
\max_{1\leq m\leq N}z_m
\leq
\operatorname{smax}_\beta\{z_m\}_{m=1}^{N}
\leq
\max_{1\leq m\leq N}z_m
+
\frac{\log N}{\beta}.
\]
Combining this inequality with
Eqs.~\eqref{eq:time-sampling-margin-bound} and
\eqref{eq:smooth-absolute-bound} yields
\begin{equation}
\left|
M^{(k)}(\alpha,t_c)
-
m^{(k)}(\alpha,t_c)
\right|
\leq
\varepsilon
+
\frac{\log N}{\beta}
+
K_t h_t.
\label{eq:ode-smooth-to-hard-margin}
\end{equation}

It remains to compare the learned and reduced-ODE sampled-smooth
margins. The map
\(
z\longmapsto\sqrt{z^2+\varepsilon^2}
\)
is \(1\)-Lipschitz, and log-sum-exp is \(1\)-Lipschitz with respect to
the \(\ell_\infty\) norm of its inputs. At any sampled time and for
any generator pair \(i<j\),
\begin{align*}
&
\left|
\left(
\widehat{\delta}^{(k)}_{\theta,i}
-
\widehat{\delta}^{(k)}_{\theta,j}
\right)
-
\left(
\delta_i^{(k)}-\delta_j^{(k)}
\right)
\right|
\\
&\qquad\leq
\left|
\widehat{\delta}^{(k)}_{\theta,i}
-
\delta_i^{(k)}
\right|
+
\left|
\widehat{\delta}^{(k)}_{\theta,j}
-
\delta_j^{(k)}
\right|
\\
&\qquad\leq
2
\sup_{t\in[0,T]}
\left\|
\widehat{x}^{(k)}_\theta(t;\alpha,t_c)
-
x^{(k)}(t;\alpha,t_c)
\right\|.
\end{align*}
Lemma~\ref{lem:hybrid-trajectory-control} therefore implies
\begin{equation}
\left|
M_\theta^{(k)}(\alpha,t_c)
-
M^{(k)}(\alpha,t_c)
\right|
\leq
2\mathcal E_x^{(k)}.
\label{eq:learned-to-ode-smooth-margin}
\end{equation}
The triangle inequality applied to
Eqs.~\eqref{eq:ode-smooth-to-hard-margin} and
\eqref{eq:learned-to-ode-smooth-margin} proves
Eq.~\eqref{eq:learned-to-hard-margin-control}.
\end{proof}

\begin{lemma}[Preservation and refinement of the initial hard-CCT branch]
\label{lem:local-cct-perturbation}
Under
Eqs.~\eqref{eq:cct-uniform-interiority}--%
\eqref{eq:cct-branch-regularity} and
\eqref{eq:cct-small-error-condition}, the first
negative-to-nonnegative cell of the learned margin belongs to the
initial reduced-model hard-CCT branch. Moreover,
\begin{equation}
\sup_{\alpha\in\mathcal A_0}
\left|
t_\theta^{\circ,(k)}(\alpha)
-
\mathrm{CCT}_k(\alpha)
\right|
\leq
\frac{\Xi^{(k)}}{\mu},
\label{eq:exact-learned-root-control}
\end{equation}
and
\begin{equation}
\sup_{\alpha\in\mathcal A_0}
\left|
\widehat{\mathrm{CCT}}_\theta^{(k)}(\alpha)
-
\mathrm{CCT}_k(\alpha)
\right|
\leq
\frac{\Xi^{(k)}}{\mu}
+
\tau_{\mathrm{root}}.
\label{eq:appendix-cct-perturbation}
\end{equation}
\end{lemma}

\begin{proof}
Fix arbitrary \(\alpha\in\mathcal A_0\), and write
\(
t_c^\star
=
t_c^\star(\alpha)
=
\mathrm{CCT}_k(\alpha),
\;
d_\Xi
=
\frac{2\Xi^{(k)}}{\mu}.
\)
We first consider \(\Xi^{(k)}>0\). The zero-error case follows by the
same argument, or by taking
\(\Xi^{(k)}\downarrow0\).

For
\(
t_c\in[t_{c,\min},t_c^\star-\rho],
\)
the pre-branch separation condition and
Lemma~\ref{lem:trajectory-margin-control} give
\[
M_\theta^{(k)}(\alpha,t_c)
\leq
m^{(k)}(\alpha,t_c)+\Xi^{(k)}
\leq
-\gamma+\Xi^{(k)}
<
0.
\]
Thus, the learned margin cannot introduce a spurious earlier crossing
outside the local hard-CCT neighborhood.

For
\(
t_c\in
[t_c^\star-\rho,t_c^\star-d_\Xi],
\)
the transversality condition gives
\[
m^{(k)}(\alpha,t_c)
\leq
-\mu(t_c^\star-t_c)
\leq
-\mu d_\Xi
=
-2\Xi^{(k)}.
\]
Consequently,
\begin{equation}
M_\theta^{(k)}(\alpha,t_c)
\leq
-\Xi^{(k)}
<
0.
\label{eq:learned-margin-left-sign}
\end{equation}
Combining these two regions yields
\[
M_\theta^{(k)}(\alpha,t_c)<0,
\qquad
t_c\leq t_c^\star-d_\Xi.
\]

Similarly, for
\(
t_c\in
[t_c^\star+d_\Xi,t_c^\star+\rho],
\)
transversality gives
\[
m^{(k)}(\alpha,t_c)
\geq
\mu(t_c-t_c^\star)
\geq
2\Xi^{(k)}.
\]
Hence
\begin{equation}
M_\theta^{(k)}(\alpha,t_c)
\geq
\Xi^{(k)}
>
0.
\label{eq:learned-margin-right-sign}
\end{equation}

By Eq.~\eqref{eq:cct-uniform-interiority} and
\eqref{eq:cct-small-error-condition},
\[
d_\Xi+h_c<\rho.
\]
Therefore,
\[
[t_c^\star+d_\Xi,\,
 t_c^\star+d_\Xi+h_c]
\subset
[t_c^\star,t_c^\star+\rho]
\subseteq
\mathcal T.
\]
Because the clearing-time grid spans \(\mathcal T\) with maximum
spacing \(h_c\), there exists a grid point \(t_{c,j}\) satisfying
\[
t_c^\star+d_\Xi
\leq
t_{c,j}
\leq
t_c^\star+d_\Xi+h_c.
\]
Equation~\eqref{eq:learned-margin-right-sign} gives
\[
M_\theta^{(k)}(\alpha,t_{c,j})>0.
\]

Every grid point at or to the left of
\(t_c^\star-d_\Xi\) has negative learned margin, whereas a positive
grid value occurs no later than
\(t_c^\star+d_\Xi+h_c\). The learned grid margins therefore possess a
first negative-to-nonnegative cell belonging to the initial hard-CCT
branch rather than to a later crossing or re-stabilization branch.

Let this first cell be
\(
[t_{c,r},t_{c,r+1}].
\)
The preceding sign arguments imply
\[
t_{c,r}
\geq
t_c^\star-d_\Xi-h_c,
\qquad
t_{c,r+1}
\leq
t_c^\star+d_\Xi+h_c.
\]
Thus, every point in this cell lies in
\[
[t_c^\star-(d_\Xi+h_c),\,
 t_c^\star+(d_\Xi+h_c)]
\subset
(t_c^\star-\rho,t_c^\star+\rho).
\]

The learned sampled-smooth margin is continuous in \(t_c\).
Since its values at the two cell endpoints are negative and
nonnegative, respectively, the intermediate value theorem gives a
root
\(
t_\theta^{\circ,(k)}(\alpha)
\in
[t_{c,r},t_{c,r+1}]
\)
satisfying
\[
M_\theta^{(k)}
\left(
\alpha,t_\theta^{\circ,(k)}(\alpha)
\right)
=
0.
\]
Because this root lies in the transversality neighborhood,
Eq.~\eqref{eq:cct-branch-regularity} implies
\begin{equation}
\mu
\left|
t_\theta^{\circ,(k)}(\alpha)-t_c^\star
\right|
\leq
\left|
m^{(k)}
\left(
\alpha,t_\theta^{\circ,(k)}(\alpha)
\right)
\right|.
\label{eq:root-transversality-bound}
\end{equation}
At the learned root,
Lemma~\ref{lem:trajectory-margin-control} gives
\begin{equation}
\begin{aligned}
|
m^{(k)}
(
&\alpha,t_\theta^{\circ,(k)}(\alpha)
)
|
\\&=
|
m^{(k)}
(
\alpha,t_\theta^{\circ,(k)}(\alpha)
)
-
M_\theta^{(k)}
\left(
\alpha,t_\theta^{\circ,(k)}(\alpha)
\right)
|
\nonumber
\leq
\Xi^{(k)}.
\end{aligned}\label{eq:hard-margin-at-learned-root}
\end{equation}
Combining
Eqs.~\eqref{eq:root-transversality-bound}  with the preceding margin estimate proves
Eq.~\eqref{eq:exact-learned-root-control}.

Finally,  the triangle inequality
gives
\begin{align*}
|
\widehat{\mathrm{CCT}}_\theta^{(k)}(\alpha)
&-
t_c^\star
|
\\&\leq
\left|
\widehat{\mathrm{CCT}}_\theta^{(k)}(\alpha)
-
t_\theta^{\circ,(k)}(\alpha)
\right|
+
\left|
t_\theta^{\circ,(k)}(\alpha)-t_c^\star
\right|
\\
&\leq
\tau_{\mathrm{root}}
+
\frac{\Xi^{(k)}}{\mu}.
\end{align*}
All constants are uniform over
\(\alpha\in\mathcal A_0\). Taking the supremum proves
Eq.~\eqref{eq:appendix-cct-perturbation}.
\end{proof}

\begin{proof}[Proof of the uniform residual-to-hard-CCT control theorem]
 Lemma
\ref{lem:trajectory-margin-control} gives the uniform hard-margin
perturbation
\[
\sup_{\substack{
\alpha\in\mathcal A_0\\
t_c\in\mathcal T
}}
\left|
M_\theta^{(k)}(\alpha,t_c)
-
m^{(k)}(\alpha,t_c)
\right|
\leq
\Xi^{(k)}.
\]
Under Eq.~\eqref{eq:cct-small-error-condition},
Lemma~\ref{lem:local-cct-perturbation} preserves the initial
hard-CCT branch and gives
\[
\sup_{\alpha\in\mathcal A_0}
\left|
\widehat{\mathrm{CCT}}_\theta^{(k)}(\alpha)
-
\mathrm{CCT}_k(\alpha)
\right|
\leq
\frac{\Xi^{(k)}}{\mu}
+
\tau_{\mathrm{root}}.
\]
\end{proof}

The clearing-time spacing \(h_c\) appears only in the sufficient
condition for locating the correct first-crossing branch. Once that
branch has been bracketed, the final numerical contribution is
controlled by \(\tau_{\mathrm{root}}\), rather than by \(h_c\).
Exact pre-fault representation and exact state chaining eliminate
separate initial-state, pre-fault, and state-interface defect terms.
The local Taylor derivatives do not enter the bound as separate error
terms; their contribution is mediated through the phase-wise
approximation and residual errors.

\begin{corollary}[Implicit sensitivity of the learned smooth boundary]
\label{cor:implicit-cct-sensitivity}
Let \(t_\theta^{\circ,(k)}(\alpha)\) be a finite, uncensored zero of
the learned sampled-smooth margin on its first CCT branch:
\[
M_\theta^{(k)}
\left(
\alpha,t_\theta^{\circ,(k)}(\alpha)
\right)
=
0.
\]
Suppose that \(M_\theta^{(k)}\) is locally \(C^1\) near this point and
that
\[
\partial_{t_c}M_\theta^{(k)}
\left(
\alpha,t_\theta^{\circ,(k)}(\alpha)
\right)
\neq0.
\]
Then the learned sampled-smooth CCT branch is locally differentiable
and satisfies
\begin{equation}
\frac{
d t_\theta^{\circ,(k)}
}{
d\alpha
}
=
-
\frac{
\partial_\alpha M_\theta^{(k)}(\alpha,t_c)
}{
\partial_{t_c}M_\theta^{(k)}(\alpha,t_c)
}
\bigg|_{
t_c=t_\theta^{\circ,(k)}(\alpha)
}.
\label{eq:appendix-cct-sensitivity}
\end{equation}
\end{corollary}

\begin{proof}
The nonzero clearing-time derivative allows the implicit function
theorem to be applied to
\[
M_\theta^{(k)}
\left(
\alpha,t_\theta^{\circ,(k)}(\alpha)
\right)
=
0.
\]
Differentiating with respect to \(\alpha\) gives
\[
\partial_\alpha M_\theta^{(k)}
+
\partial_{t_c}M_\theta^{(k)}
\frac{
d t_\theta^{\circ,(k)}
}{
d\alpha
}
=
0,
\]
which proves Eq.~\eqref{eq:appendix-cct-sensitivity}.
\end{proof}

The corollary concerns the exact zero of the learned sampled-smooth
margin. In numerical evaluation, derivatives are evaluated at the
root-refined approximation
\(\widehat{\mathrm{CCT}}_\theta^{(k)}\), whose distance from
\(t_\theta^{\circ,(k)}\) is bounded by
\(\tau_{\mathrm{root}}\). Derivatives are computed through the complete
computational graph and only where the local phase selection is
\(C^1\). Censored points, poorly conditioned roots, and event-grid
coincidences are excluded. The corollary does not by itself bound the
error relative to the sensitivity of either the reduced-model hard
CCT or the full-network DAE boundary.

\section{Boundary Evaluation and DAE Numerical Reliability}
\label{app:numerical-reliability}

\paragraph{Root-refined boundary protocol.}

All learned and reduced-ODE CCT boundaries use a common first-crossing
extraction protocol. For each fixed severity \(\alpha\), the learned
margin is first evaluated on the clearing-time grid
\[
\mathcal T_{h_c}
=
\left\{
t_{c,0},\ldots,t_{c,J}
\right\},
\qquad
h_c=10^{-2}\,\mathrm{s}.
\]
We select the first cell satisfying
\[
M_\theta^{(k)}
\left(
\alpha,t_{c,j}
\right)
<0,
\qquad
M_\theta^{(k)}
\left(
\alpha,t_{c,j+1}
\right)
\geq0.
\]
The coarse grid is used only to identify the initial
stable-to-unstable transition. Within the selected bracket, the
learned margin is recomputed at continuous clearing-time queries, and
Brent refinement returns
\(\widehat{\mathrm{CCT}}_\theta^{(k)}(\alpha)\) with absolute
tolerance
\(
\tau_{\mathrm{root}}
=
10^{-5}\,\mathrm{s}
\)
and relative tolerance \(10^{-12}\). Consequently, \(h_c\) controls
successful branch localization but does not determine the final
clearing-time resolution.

The reduced-ODE hard CCT is extracted independently using the same
first-branch protocol. Its hard margin is evaluated from trajectories
with temporal spacing
\(
h_t=10^{-3}\,\mathrm{s},
\)
ODE tolerances
\(\mathrm{rtol}=10^{-9}\) and
\(\mathrm{atol}=10^{-11}\), and an independently refined
stable-to-unstable root. Thus, the primary reduced-model CCT error does
not compare two boundaries interpolated from the same clearing-time
grid.

At a finite and locally transverse exact learned root, the
implicit-function theorem gives
\[
\frac{
d t_\theta^{\circ,(k)}
}{
d\alpha
}
=
-
\frac{
\partial_\alpha M_\theta^{(k)}(\alpha,t_c)
}{
\partial_{t_c}M_\theta^{(k)}(\alpha,t_c)
}
\bigg|_{
t_c=t_\theta^{\circ,(k)}(\alpha)
}.
\]
In numerical evaluation, the derivatives are evaluated at the
root-refined approximation
\(\widehat{\mathrm{CCT}}_\theta^{(k)}(\alpha)\), for which
\[
M_\theta^{(k)}
\left(
\alpha,
\widehat{\mathrm{CCT}}_\theta^{(k)}(\alpha)
\right)
\approx0.
\]
The derivatives are obtained through the complete deployed
computational graph, with Taylor-coefficient detachment disabled.
The corresponding finite-difference diagnostic uses independently
root-refined learned CCT values at adjacent severity nodes:
\[
D_{\mathrm{FD}}(\alpha_i)
=
\frac{
\widehat{\mathrm{CCT}}_\theta^{(k)}(\alpha_{i+1})
-
\widehat{\mathrm{CCT}}_\theta^{(k)}(\alpha_{i-1})
}{
\alpha_{i+1}-\alpha_{i-1}
}.
\]
Rows with censored neighboring boundaries, an ill-conditioned
\(\partial_{t_c}M_\theta^{(k)}\), or a clearing event coinciding with
the trajectory-time grid are excluded from the sensitivity
comparison.

\paragraph{Margin-approximation diagnostics.}

We separately quantify three boundary discrepancies induced by margin
approximation and surrogate error:
\[
E_{\mathrm{sm}}(\beta)
=
\operatorname{MAE}
\left(
\mathrm{CCT}_{\mathrm{ODE}}^{\mathrm{smooth}}(\beta),
\mathrm{CCT}_{\mathrm{ODE}}^{\mathrm{hard}}
\right),
\]
\[
E_{\mathrm{sur}}(\beta)
=
\operatorname{MAE}
\left(
\widehat{\mathrm{CCT}}_\theta^{\mathrm{smooth}}(\beta),
\mathrm{CCT}_{\mathrm{ODE}}^{\mathrm{smooth}}(\beta)
\right),
\]
and
\[
E_{\mathrm{end}}(\beta)
=
\operatorname{MAE}
\left(
\widehat{\mathrm{CCT}}_\theta^{\mathrm{smooth}}(\beta),
\mathrm{CCT}_{\mathrm{ODE}}^{\mathrm{hard}}
\right).
\]
The first isolates smooth-margin approximation, the second isolates
the neural-surrogate discrepancy under the same smooth margin, and
the third measures the end-to-end reduced-model boundary error. These
quantities are related by the triangle inequality but are not assumed
to form an additive error decomposition.

\begin{figure}[t]
\centering
\includegraphics[width=\linewidth]
{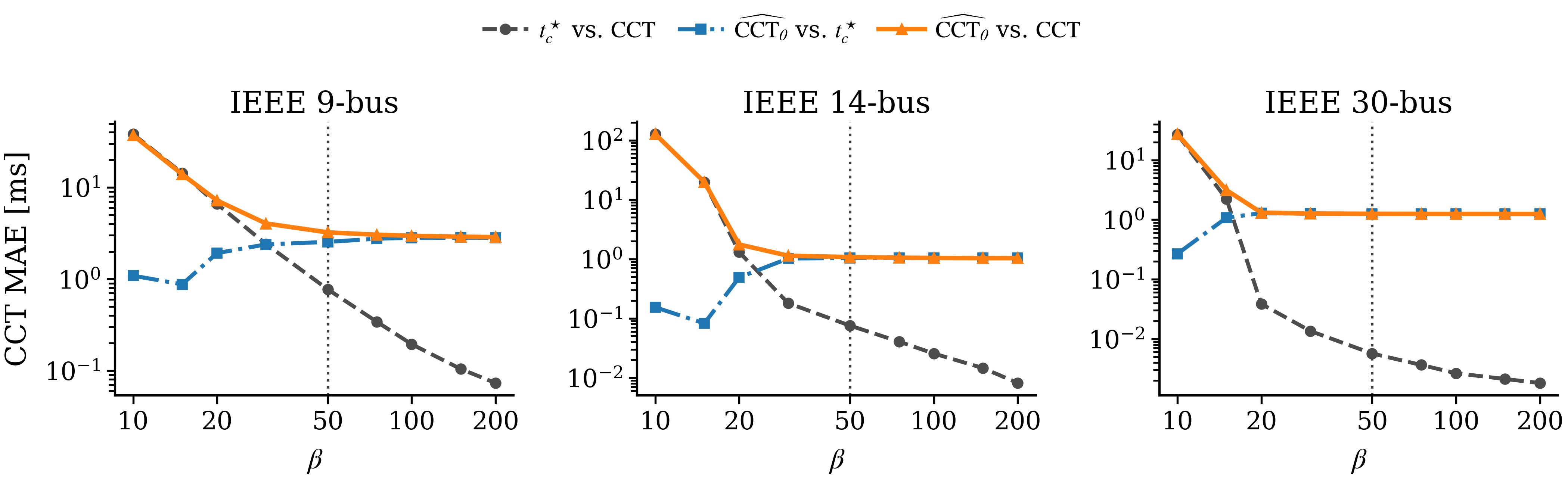}
\caption{Sensitivity of root-refined CCT MAE to the smooth-maximum
parameter \(\beta\) on the IEEE 9-, 14-, and 30-bus systems. The gray
curve reports the reduced-ODE smooth-to-hard discrepancy
\(E_{\mathrm{sm}}\); the blue curve reports the ES-PINN-to-ODE
surrogate discrepancy \(E_{\mathrm{sur}}\) under the same smooth
margin; and the red curve reports the end-to-end ES-PINN smooth-to-ODE
hard error \(E_{\mathrm{end}}\). The vertical dotted line marks the
default \(\beta=50\). Each point is evaluated on the finite,
uncensored severity rows of the corresponding system.}
\label{fig:beta-sensitivity}
\end{figure}

Figure~\ref{fig:beta-sensitivity} shows that small values of \(\beta\)
produce substantial smooth-maximum bias, especially in the IEEE
14-bus system. At the default \(\beta=50\), the ODE smooth-to-hard CCT
MAEs are
\(0.771\,\mathrm{ms}\),
\(0.075\,\mathrm{ms}\), and
\(0.006\,\mathrm{ms}\)
for the IEEE 9-, 14-, and 30-bus systems, respectively. The
corresponding end-to-end ES-PINN smooth-to-ODE hard errors are
\(3.232\,\mathrm{ms}\),
\(1.077\,\mathrm{ms}\), and
\(0.326\,\mathrm{ms}\).
Increasing \(\beta\) beyond \(50\) further reduces the deterministic
smooth-to-hard discrepancy but changes the end-to-end error only
modestly, indicating that the remaining error is dominated by the
learned surrogate rather than margin smoothing.

We also tested the temporal sampling used by the high-resolution ODE
hard-margin extractor. Across the \(23\) finite Case-30 boundary rows,
changing the default \(h_t=1\,\mathrm{ms}\) to
\(0.5\,\mathrm{ms}\) or \(2\,\mathrm{ms}\) changed the root-refined
hard CCT by at most
\(1.6\times10^{-6}\,\mathrm{ms}\).
The temporal-sampling contribution is therefore negligible at the
reported precision, and no separate \(h_t\)-convergence table is
required.

\paragraph{Full-network DAE reliability.}

The full-network reference is computed using a fixed-step implicit
trapezoidal DAE solver. At each time step, the coupled nonlinear
equations are solved by Newton iteration with tolerance \(10^{-8}\)
and a maximum of \(15\) iterations. Fault application and clearing
are aligned with the integration grid, and an algebraic correction is
applied immediately after network switching so that the post-event
state satisfies the cleared-network algebraic constraints. Hard CCTs
are extracted by an adaptive first-crossing search over clearing time.
A severity is labeled right-censored when no stable-to-unstable
transition occurs within the prescribed clearing-time interval.
Rollouts with nonfinite states, singular Newton systems, or failed
nonlinear convergence are recorded as solver failures rather than
interpreted as stability outcomes.

The targeted DAE sweep evaluates \(21\) severity values per system.
It yields \(5\), \(4\), and \(12\) finite, uncensored CCTs for the IEEE
9-, 14-, and 30-bus systems, respectively. The remaining \(16\),
\(17\), and \(9\) severities are right-censored, with no unresolved
solver failures in the final sweep.
Table~\ref{tab:dae_timestep_convergence} examines time-step
sensitivity at one finite interior severity per system.

\begin{table}[t]
\centering
\caption{Full-network DAE time-step refinement at one finite,
uncensored interior severity per system. Each hard CCT is extracted
using event-grid-aligned adaptive DAE rollouts. The final column
reports the maximum difference among the three tested step sizes
relative to the finest \(2.5\)-ms result.}
\label{tab:dae_timestep_convergence}
\scriptsize
\setlength{\tabcolsep}{2.7pt}
\resizebox{\columnwidth}{!}{%
\begin{tabular}{lccccc}
\toprule
System & \(\alpha\)
& CCT, 10 ms (s)
& CCT, 5 ms (s)
& CCT, 2.5 ms (s)
& Max. diff. (ms) \\
\midrule
IEEE 9-bus  & 0.90 & 0.66021 & 0.66509 & 0.66754 & 7.33 \\
IEEE 14-bus & 0.90 & 1.03020 & 1.03015 & 1.03011 & 0.09 \\
IEEE 30-bus & 0.70 & 0.68034 & 0.68516 & 0.68757 & 7.23 \\
\bottomrule
\end{tabular}%
}
\end{table}

The IEEE 14-bus boundary is effectively insensitive to the tested
time steps, changing by only \(0.09\,\mathrm{ms}\). For the IEEE 9-
and 30-bus systems, the successive CCT changes decrease by
approximately a factor of two when the time step is halved, indicating
consistent convergence under time-step refinement. Nevertheless, the
corresponding \(10\)-ms and \(2.5\)-ms results differ by
\(7.33\,\mathrm{ms}\) and \(7.23\,\mathrm{ms}\), respectively.
We therefore use the \(2.5\)-ms DAE results for the quantitative
full-network comparisons and interpret them as finite-resolution
external validation rather than numerically exact physical CCTs.

\section{Additional Experimental Results}
\label{app:additional-results}

\paragraph{Additional Case-30 diagnostics.}

Tables~\ref{tab:case30_restore_matched} and
\ref{tab:case30_restore_extrapolation} provide two supplementary views
of the IEEE 30-bus restore-clearing benchmark. The matched-grid results
measure reconstruction on parameter pairs used during training and are
included only as diagnostics of optimization and representation
capacity. The mild-extrapolation results evaluate severities beyond
the training interval while retaining clearing times within the
training range.

For a rotor-angle trajectory, define the instantaneous maximum
pairwise separation as
\[
s(t)
=
\max_{i<j}
\left|
\delta_i(t)-\delta_j(t)
\right|.
\]
Sep.\ MAE averages
\(
|\widehat{s}_\theta(t)-s(t)|
\)
over the evaluated parameter pairs and trajectory-time samples. For
the supplementary label-based diagnostic, mIoU is the mean of the
stable- and unstable-class intersection-over-union values induced by
the sampled-smooth surrogate and reduced-ODE margins.

\begin{table}[t]
\centering
\caption{Matched-grid diagnostics for the IEEE 30-bus
restore-clearing benchmark. Values are mean \(\pm\) standard deviation
over five seeds. Trajectory and separation metrics aggregate all
\(2{,}601\) training-grid parameter pairs, while mIoU compares
sampled-smooth stability labels on the same grid. These results measure
reconstruction on sampled nodes and are not used as the primary
generalization claim. Lower is better except for mIoU.}
\label{tab:case30_restore_matched}
\scriptsize
\setlength{\tabcolsep}{2.0pt}
\resizebox{\columnwidth}{!}{%
\begin{tabular}{@{}lcccc@{}}
\toprule
Metric
& Vanilla PINN
& Enhanced XPINN
& PI-DeepONet
& ES-PINN \\
\midrule

\multicolumn{5}{@{}l}{\textit{Trajectory accuracy}}\\

\(\mathrm{RMSE}_{\delta}\) (rad)
& \((9.42\pm2.92)\times10^{-2}\)
& \((5.05\pm0.37)\times10^{-3}\)
& \((2.57\pm1.17)\times10^{-1}\)
& \(\mathbf{(2.95\pm0.31)\times10^{-3}}\) \\

\(\mathrm{RMSE}_{\omega}\) (p.u.)
& \((1.29\pm0.43)\times10^{-3}\)
& \((6.80\pm0.50)\times10^{-5}\)
& \((2.74\pm0.27)\times10^{-3}\)
& \(\mathbf{(4.10\pm0.40)\times10^{-5}}\) \\

Sep.\ MAE (\(^{\circ}\))
& \(5.02\pm1.96\)
& \(0.203\pm0.013\)
& \(9.97\pm2.85\)
& \(\mathbf{0.130\pm0.016}\) \\

\addlinespace[2pt]
\multicolumn{5}{@{}l}{\textit{Sampled stability-map accuracy}}\\

mIoU
& \(0.9929\pm0.0058\)
& \(\mathbf{1.0000\pm0.0000}\)
& \(0.9986\pm0.0008\)
& \(\mathbf{1.0000\pm0.0000}\) \\

\bottomrule
\end{tabular}%
}
\end{table}

Matched-grid CCT MAE and P95 errors are intentionally omitted. These
near-zero same-grid quantities provide little information about
generalization and do not follow the independently root-refined
hard-CCT protocol used for the primary boundary comparison.

\begin{table}[t]
\centering
\caption{Mild severity-extrapolation accuracy for the IEEE 30-bus
restore-clearing benchmark. Values are mean \(\pm\) standard deviation
over five seeds. Trajectory and separation metrics aggregate
\(1{,}010\) parameter pairs with
\(\alpha\in[1.02,1.20]\) and \(101\) clearing times within the training
clearing-time interval. Hard-CCT errors compare root-refined learned
smooth-margin boundaries with independently root-refined reduced-ODE
hard-margin boundaries using \(1\)-ms trajectory-time sampling on the
finite, uncensored severity rows. The mIoU comparison uses
sampled-smooth stability labels. Lower is better except for mIoU.}
\label{tab:case30_restore_extrapolation}
\scriptsize
\setlength{\tabcolsep}{2.0pt}
\resizebox{\columnwidth}{!}{%
\begin{tabular}{@{}lcccc@{}}
\toprule
Metric
& Vanilla PINN
& Enhanced XPINN
& PI-DeepONet
& ES-PINN \\
\midrule

\multicolumn{5}{@{}l}{\textit{Trajectory accuracy}}\\

\(\mathrm{RMSE}_{\delta}\) (rad)
& \((3.70\pm0.95)\times10^{-1}\)
& \((2.41\pm0.99)\times10^{-1}\)
& \((8.56\pm0.52)\times10^{-1}\)
& \(\mathbf{(1.78\pm0.42)\times10^{-1}}\) \\

\(\mathrm{RMSE}_{\omega}\) (p.u.)
& \((2.99\pm0.65)\times10^{-3}\)
& \((1.72\pm0.29)\times10^{-3}\)
& \((5.45\pm0.16)\times10^{-3}\)
& \(\mathbf{(1.43\pm0.32)\times10^{-3}}\) \\

Sep.\ MAE (\(^{\circ}\))
& \(21.35\pm6.42\)
& \(13.42\pm7.68\)
& \(59.61\pm33.50\)
& \(\mathbf{10.60\pm1.54}\) \\

\addlinespace[2pt]
\multicolumn{5}{@{}l}{\textit{Root-refined boundary accuracy}}\\

Hard-CCT MAE (ms)
& \(14.63\pm4.52\)
& \(5.31\pm1.14\)
& \(12.02\pm1.36\)
& \(\mathbf{1.28\pm0.51}\) \\

Hard-CCT P95 err. (ms)
& \(24.20\pm5.69\)
& \(8.78\pm5.09\)
& \(22.55\pm2.15\)
& \(\mathbf{3.94\pm2.80}\) \\

mIoU
& \(0.9382\pm0.0173\)
& \(0.9943\pm0.0054\)
& \(0.9480\pm0.0044\)
& \(\mathbf{0.9952\pm0.0052}\) \\

\bottomrule
\end{tabular}%
}
\end{table}

The matched-grid results confirm that ES-PINN accurately reconstructs
the sampled parameter nodes without relying on a capacity advantage.
More importantly, the mild-extrapolation results show that its
trajectory and root-refined boundary errors remain lower than those
of the three comparison methods outside the training severity range.

Figure~\ref{fig:case30_restore_full_trajectory} complements the
critical-pair diagnostic  by
displaying the rotor-angle trajectories of all generators for the same
supercritical scenario.

\begin{figure}[t]
\centering
\includegraphics[width=\columnwidth]
{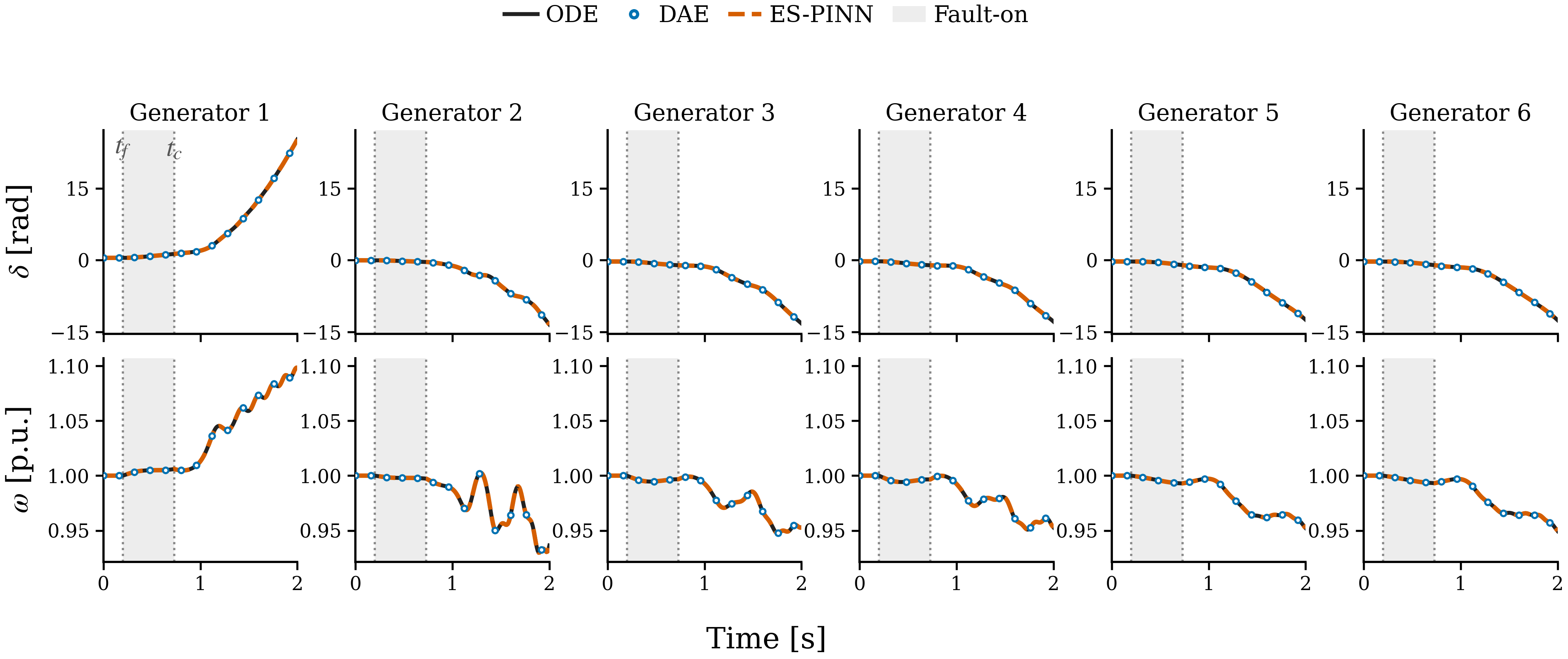}
\caption{Full-generator rotor-angle diagnostic for the IEEE 30-bus
restore-clearing scenario. Reduced
ODE, full-network DAE, and ES-PINN trajectories are shown for every
generator at \(\alpha=0.93\) and
\(t_c=0.73\,\mathrm{s}\). Colors identify the trajectory source, and
shading denotes the fault-on phase.}
\label{fig:case30_restore_full_trajectory}
\end{figure}

\paragraph{Root-based boundary sensitivity.}

We additionally evaluate whether the implicit derivative of each
learned CCT boundary is consistent with finite differences of that
same learned boundary. Following the protocol in
Sec.~\ref{app:numerical-reliability}, the coarse clearing-time grid
first identifies the initial sign-changing cell, after which the
learned margin is root-refined within that cell. At the corresponding
exact learned zero, the implicit derivative is
\[
\frac{
d t_\theta^{\circ,(k)}
}{
d\alpha
}
=
-
\frac{
\partial_\alpha M_\theta^{(k)}(\alpha,t_c)
}{
\partial_{t_c}M_\theta^{(k)}(\alpha,t_c)
}
\bigg|_{
t_c=t_\theta^{\circ,(k)}(\alpha)
}.
\]
In numerical evaluation, this expression is evaluated through the
complete computational graph at the root-refined approximation
\(\widehat{\mathrm{CCT}}_\theta^{(k)}(\alpha)\). Centered finite
differences use independently root-refined learned CCTs at adjacent
severity nodes.

This experiment is an internal differentiability diagnostic rather
than a comparison with the physical CCT derivative. A row is
considered valid only when the center and both neighboring severity
nodes have finite learned roots, the clearing-time margin derivative
is well-conditioned, and the refined clearing event does not coincide
with the trajectory-time grid.

\begin{table}[t]
\centering
\caption{Root-based implicit--finite-difference consistency of learned
CCT sensitivities for the IEEE 30-bus restore-clearing benchmark.
Values are mean \(\pm\) standard deviation over five seeds. The
implicit derivative is evaluated through the complete computational
graph at each root-refined learned zero, while the centered finite
difference uses adjacent root-refined CCTs extracted from the same
learned margin. Each method has \(23\) finite, well-conditioned
boundary points per seed. Lower MAE and higher correlation are better.}
\label{tab:case30_restore_sensitivity_metrics}
\scriptsize
\setlength{\tabcolsep}{2.4pt}
\resizebox{\columnwidth}{!}{%
\begin{tabular}{@{}lcc@{}}
\toprule
Method
& Implicit--FD MAE (\(\mathrm{ms}/\alpha\))
& Correlation \\
\midrule
Vanilla PINN
& \(25.67\pm4.97\)
& \(0.988\pm0.004\) \\

Enhanced XPINN
& \(23.56\pm7.93\)
& \(0.994\pm0.002\) \\

PI-DeepONet
& \(26.16\pm11.71\)
& \(0.976\pm0.005\) \\

ES-PINN
& \(\mathbf{15.24\pm0.88}\)
& \(\mathbf{0.999\pm0.001}\) \\
\bottomrule
\end{tabular}%
}
\end{table}

All four methods exhibit high implicit--finite-difference correlation,
but ES-PINN provides the closest quantitative agreement. Relative to
the strongest comparison method, Enhanced XPINN, it reduces the
implicit--finite-difference MAE by approximately \(35\%\), while
increasing the mean correlation from \(0.994\) to \(0.999\). This
result supports the numerical consistency of the differentiable,
root-refined ES-PINN boundary; it is not, by itself, a comparison with
the derivative of the reduced-model hard CCT or the physical DAE
boundary.

\end{document}